\pdfoutput=1

\documentclass[11pt]{article}

\usepackage{acl}

\usepackage{times}
\usepackage{latexsym}

\usepackage[T1]{fontenc}
\usepackage[utf8]{inputenc}
\usepackage{microtype}

\usepackage{inconsolata}

\usepackage{amsmath}
\usepackage{hyperref}
\usepackage{cleveref} 

\usepackage{amsfonts}
\usepackage{apxproof}
\usepackage{booktabs}
\usepackage{cancel}
\usepackage{fontawesome}
\usepackage{graphicx}
\usepackage{natbib}
\usepackage{thmtools}
\usepackage{thm-restate}
\usepackage{tikz}

\usepackage{xspace}
\usepackage{amsthm}
\usepackage{bm}
\usepackage{bbm}
\usepackage{multirow}
\usepackage{multicol}
\usepackage{subcaption}
\usepackage[inline]{enumitem}
\usepackage{algorithm}
\usepackage[noend]{algpseudocode} 
\usepackage{scalerel}
\usepackage{mathtools}

\usepackage[disable]{todonotes}
\crefname{section}{\S}{\S\S} %
\Crefname{section}{\S}{\S\S} %
\crefname{lemma}{Lemma}{Lemmata}
\crefname{table}{Table}{Tables}
\crefname{figure}{Figure}{Figures}
\crefname{algorithm}{Algorithm}{}
\crefname{equation}{Eq.}{Eqs.}
\crefname{appendix}{App.}{Appendices}
\crefformat{section}{\S#2#1#3}  

\newtheorem{definition}{Definition}[section]

\newtheorem{lemma}{Lemma}[section]

\newtheorem{remark}{Remark}

\theoremstyle{definition}
\crefname{thm}{Theorem}{Theorems}
\crefname{remark}{Remark}{Remarks}
\crefname{defin}{Def.}{Defs.}

\makeatletter
\newcommand*\iftodonotes{\if@todonotes@disabled\expandafter\@secondoftwo\else\expandafter\@firstoftwo\fi} 
\makeatother

\definecolor{ETHBlue}{RGB}{33,92,175}	%
\definecolor{ETHGreen}{RGB}{98,115,19}		%
\definecolor{ETHPurpleDark}{RGB}{140,10,89}	%
\definecolor{ETHPurple}{RGB}{163,7,116}	%
\definecolor{ETHGray}{RGB}{111,111,111}	%
\definecolor{ETHRed}{RGB}{183,53,45}	%
\definecolor{ETHPetrol}{RGB}{0,120,148}	%
\definecolor{ETHBronze}{RGB}{142,103,19}	%

\colorlet{ETHdarkblue}{ETHBlue!80!black}
\colorlet{ETHdarkgreen}{ETHGreen!80!black}
\colorlet{ETHpink}{ETHPurple}
\colorlet{ETHgray}{ETHGray}
\colorlet{ETHred}{ETHRed}
\colorlet{ETHgreenblue}{ETHPetrol}
\colorlet{ETHbrown}{ETHBronze}

\definecolor{TextBlack}{RGB}{51,51,51}
\definecolor{BackgroundWhite}{RGB}{255,255,255}

\definecolor{AccentBlue}{RGB}{0,122,204}
\definecolor{LightBlue}{RGB}{173,216,230}
\definecolor{DarkBlue}{RGB}{0,51,102}

\definecolor{AccentGreen}{RGB}{70,160,73}
\definecolor{LightGreen}{RGB}{144,238,144}
\definecolor{DarkGreen}{RGB}{0,100,0}

\definecolor{AccentRed}{RGB}{255,0,0}
\definecolor{LightRed}{RGB}{255,99,71}
\definecolor{DarkRed}{RGB}{139,0,0}

\definecolor{AccentOrange}{RGB}{255,165,0}
\definecolor{LightOrange}{RGB}{255,204,153}
\definecolor{DarkOrange}{RGB}{255,140,0}

\definecolor{NeutralLightGray}{RGB}{204,204,204}
\definecolor{NeutralMediumGray}{RGB}{102,102,102}

\definecolor{NoteYellow}{RGB}{255,255,0}

\definecolor{DiversePurple}{RGB}{128,0,128}
\definecolor{DiverseTeal}{RGB}{0,128,128}
\definecolor{DiverseOlive}{RGB}{128,128,0}
\definecolor{DiverseCyan}{RGB}{0,128,192}
\definecolor{DiverseMagenta}{RGB}{192,0,128}

\colorlet{MacroColor}{ETHGreen}
\colorlet{MACROCOLOR}{MacroColor}
\usepackage{amsmath,amsfonts,bm}

\usepackage[overleftharpoonup, overrightharpoonup]{overarrows}

\newcommand{\justification}[1]{%
    \refstepcounter{equation}%
    \tag{\textcolor{black!50}{\footnotesize{#1},} \theequation}
}

\newcommand{\mymacro}[1]{{#1}}

\newcommand{\defn}[1]{\textbf{#1}}
\newcommand{\defeq}{\mathrel{\stackrel{\textnormal{\tiny def}}{=}}}

\newcommand{\sigmoid}{{\mymacro{\sigma}}}

\newcommand{\ind}[1]{\mathbbm{1} \left\{ #1 \right\}}

\newcommand{\bijection}{\mymacro{f}}

\newcommand{\counts}{\mymacro{N}}
\newcommand{\countsFun}[1]{{\mymacro{\counts}\!\left(#1\right)}}

\newcommand{\pfsaAcr}{{\mymacro{PFSA}}\xspace}
\newcommand{\dpfsaAcr}{{\mymacro{DPFSA}}\xspace}

\newcommand{\automaton}{{\mymacro{\mathcal{A}}}}
\newcommand{\wfsa}{{\mymacro{\automaton}}}

\newcommand{\stateq}{{\mymacro{q}}}

\newcommand{\states}{{\mymacro{Q}}}

\newcommand{\trans}{{\mymacro{\delta}}}

\newcommand{\prevq}{{\mymacro{\iota }}}
\newcommand{\prevqFun}[1]{{\mymacro{\prevq}}\!\left(#1\right)}
\newcommand{\nextq}{{\mymacro{\varphi }}}
\newcommand{\nextqFun}[1]{{\mymacro{\nextq}}\!\left(#1\right)}
\newcommand{\pathweightFun}[1]{{\mymacro{\boldsymbol{w}(#1)}}}
\newcommand{\innerpathweightFun}[1]{\mymacro{\overline{\boldsymbol{w}}(#1)}}
\newcommand{\apath}{{\mymacro{\boldsymbol \pi}}}
\newcommand{\pathlen}{{\mymacro{N}}}
\newcommand{\paths}{{\mymacro{\Pi}}}

\newcommand{\initf}{{\mymacro{\lambda}}}
\newcommand{\finalf}{{\mymacro{\rho}}}
\newcommand{\initfFun}[1]{{\mymacro{\initf\!\left(#1\right)}}}
\newcommand{\finalfFun}[1]{{\mymacro{\finalf\!\left(#1\right)}}}

\newcommand{\qinit}{{\mymacro{q_{\iota}}}}

\newcommand{\wfsatuple}{{\mymacro{\left( \alphabet, \states, \trans, \initf, \finalf \right)}}}

\newcommand{\meanStrLen}{{\mymacro{\mu}}}

\newcommand{\entropy}{{\mymacro{\mathrm{H}}}}
\newcommand{\entropyFun}[1]{{\mymacro{\entropy}(#1)}}
\newcommand{\entropyhat}{{\mymacro{\mathrm{\widehat{H}}}}}
\newcommand{\entropyhatFun}[1]{{\mymacro{\entropyhat}\!\left(#1\right)}}
\newcommand{\mlocalentropy}{{\mymacro{\mathrm{H}_{m}}}}

\newcommand{\edge}[4]{{\mymacro{#1 \xrightarrow{#2 / #3} #4}}}
\newcommand{\tranMtx}{{\mymacro{\mM}}}
\newcommand{\etranMtx}{{\mymacro{\emM}}}
\newcommand{\sTranMtx}[1]{{\mymacro{\tranMtx}^{(#1)}}}
\newcommand{\esTranMtx}[1]{{\mymacro{\etranMtx}^{(#1)}}}
\newcommand{\yield}{{\mymacro{\textbf{s}}}}
\newcommand{\yieldFun}[1]{{\mymacro{\yield}(#1)}}

\newcommand{\bigO}{{\mymacro{\mathcal{O}}}}
\newcommand{\bigOFun}[1]{{\mymacro{\bigO}\!\left(#1\right)}}
\newcommand{\idx}{{\mymacro{n}}}

\newcommand{\nstates}{{\mymacro{|\states|}}}
\newcommand{\nsymbols}{{\mymacro{|\alphabet|}}}

\newcommand{\tstep}{{\mymacro{t}}}

\newcommand{\vlambda}{{\mymacro{\boldsymbol{\lambda}}}}

\newcommand{\prob}{{\mymacro{\textnormal{\textbf{Pr}}}}}
\newcommand{\probFun}[1]{{\mymacro{\prob}}(#1)}
\newcommand{\pdens}{{\mymacro{p}}}
\newcommand{\qdens}{{\mymacro{q}}}

\newcommand{\pLM}{\mymacro{\pdens}}
\newcommand{\pLMwfsa}{\mymacro{\pLM_{\scaleto{\wfsa}{4pt}}}}
\newcommand{\pLMFun}[1]{\mymacro{\pLM}\!\left(#1\right)}
\newcommand{\pLMwfsaFun}[1]{\mymacro{\pLMwfsa}(#1)}

\newcommand{\qLM}{\mymacro{\qdens}}
\newcommand{\qLMFun}[1]{\mymacro{\qLM}\!\left(#1\right)}
\newcommand{\qLMprefix}{\mymacro{\overrightharpoonup{\qdens}}}

\newcommand{\pLMhat}{\mymacro{\widehat{\pLM}}}
\newcommand{\pLMhatFun}[1]{\mymacro{\pLMhat}\!\left(#1\right)}

\newcommand{\pLMprefix}{\mymacro{\overrightharpoonup{\pdens}}}
\newcommand{\pLMprefixFun}[1]{{\mymacro{\pLMprefix}\!\left(#1\right)}}
\newcommand{\prefixnorm}{{\mymacro{\overrightharpoonup{Z}}}}
\newcommand{\pLMwfsaprefix}{\mymacro{\overrightharpoonup{\pdens}_{\scaleto{\wfsa}{4pt}}}}
\newcommand{\pLMwfsaprefixFun}[1]{{\mymacro{\pLMwfsaprefix}\!\left(#1\right)}}

\newcommand{\pLMinfix}{%
  \mymacro{\overset{%
    \smash{
      \mathrlap{\xleftharpoonup{\phantom{\pdens}}}%
      \xrightharpoonup{\phantom{\pdens}}%
    }
  }{\smash[t]{\pdens}\rule{0pt}{0.06cm}}%
}}
\newcommand{\pLMinfixFun}[1]{{\mymacro{\pLMinfix}\!\left(#1\right)}}
\newcommand{\pLMwfsainfix}{\mymacro{\pLMinfix_{\scaleto{\wfsa}{4pt}}}}
\newcommand{\pLMwfsainfixFun}[1]{{\mymacro{\pLMwfsainfix}\!\left(#1\right)}}

\newcommand{\infixnorm}{%
  \mymacro{\overset{%
    \smash{
      \mathrlap{\xleftharpoonup{\phantom{Z}}}%
      \xrightharpoonup{\phantom{Z}}%
    }\rule{0pt}{0.13cm}
  }{\smash[t]{Z}\rule{0pt}{0.16cm}}%
}}

\newcommand{\kleene}[1]{#1^\ast}
\newcommand{\alphabet}{{\mymacro{\Sigma}}}
\newcommand{\mAlphabet}{{\mymacro{\alphabet^{m - 1}}}}
\newcommand{\klAlphabet}{{\mymacro{\kleene{\alphabet}}}}
\newcommand{\str}{{\mymacro{\boldsymbol{y}}}}
\newcommand{\ctx}{{\mymacro{\boldsymbol{c}}}}

\newcommand{\strlt}{{\mymacro{\str_{<\tstep}}}}

\newcommand{\strlen}{{\mymacro{T}}}
\newcommand{\bos}{{\mymacro{\textsc{bos}}}}
\newcommand{\eos}{{\mymacro{\textsc{eos}}}}
\newcommand{\eosalphabet}{{\mymacro{\overline{\alphabet}}}}

\newcommand{\ngram}{{\mymacro{\textit{n}-gram}}\xspace}
\newcommand{\mlocal}{{\mymacro{\textit{m}-local}}\xspace}
\newcommand{\Mlocal}{{\mymacro{\textit{M}-local}}\xspace}

\newcommand{\sym}{{\mymacro{y}}}

\newcommand{\rvstr}{{\mymacro{\boldsymbol{Y}}}}
\newcommand{\rvctx}{{\mymacro{\boldsymbol{C}}}}

\newcommand{\rvpathpfx}{{\mymacro{\overrightharpoonup{\Pi}}}}
\newcommand{\rvpathinfx}{%
  \mymacro{\overset{%
    \smash{
      \mathrlap{\xleftharpoonup{\phantom{\Pi}}}%
      \xrightharpoonup{\phantom{\Pi}}%
    }\rule{0pt}{0.11cm}
  }{\smash[t]{\Pi}\rule{0pt}{0.17cm}}%
}}

\newcommand{\rvNextSymbol}{\mymacro{\overline{Y}}}
\newcommand{\rvPrefix}{\mymacro{\overrightharpoonup{\smash[t]{\bm{Y}}\rule{0pt}{0.22cm}}}}
\newcommand{\rvContextNextSymbol}{\mymacro{\rvNextSymbol}}

\newcommand{\symT}{{\mymacro{\sym_{\strlen}}}}

\def\1{\mymacro{\mathbf{1}}}
\newcommand{\dataset}{{\mymacro{\sD}}}
\newcommand{\datasetSize}{{\mymacro{N}}}
\newcommand{\totalsymcount}{\mymacro{S}}

\def\vzero{\mymacro{\bm{0}}}

\def\vrho{\mymacro{\bm{\rho}}}

\def\mE{\mymacro{\bm{E}}}

\def\mI{\mymacro{\bm{I}}}

\def\mM{\mymacro{\bm{M}}}

\DeclareMathAlphabet{\mathsfit}{\encodingdefault}{\sfdefault}{m}{sl}
\SetMathAlphabet{\mathsfit}{bold}{\encodingdefault}{\sfdefault}{bx}{n}

\def\symy{{\mymacro{y}}}

\def\sD{{\mymacro{\mathbb{D}}}}

\def\emE{\mymacro{E}}

\def\emM{\mymacro{M}}

\newcommand{\E}{\mathbb{E}}

\newcommand{\R}{\mathbb{R}}

\newcommand{\KL}{{\mymacro{D_{\mathrm{KL}}}}}

\newcommand{\KLhat}{{\mymacro{\widehat{D}_{\mathrm{KL}}}}}

\newcommand{\normltwo}{L^2}

\newcommand{\thestring}{\mymacro{w}}

\newcommand{\thesymboli}[1]{\mymacro{\thestring_{#1}}}
\newcommand{\thelength}{\mymacro{n}}

\newcommand{\thetimestep}{\mymacro{t}}

\newcommand{\thelayerno}{\mymacro{\ell}}
\newcommand{\thenumlayers}{\mymacro{L}}
\newcommand{\thehiddensize}{\mymacro{d}}

\newcommand{\realset}{\mymacro{\R}}

\newcommand{\vecvar}[1]{\bm{#1}}
\newcommand{\matvar}[1]{\bm{#1}}

\newcommand{\logistic}[1]{\sigmoid(#1)}
\newcommand{\elementwisemultiply}{\mathbin{\mymacro{\odot}}}

\newcommand{\vectorparam}[2]{\vecvar{#1}_{\mathrm{#2}}}
\newcommand{\vectorparaml}[3]{\vectorparam{#1}{#2}^{(#3)}}

\newcommand{\lineartransforml}[3]{\weightparaml{#1}{#2} #3}

\newcommand{\affinel}[3]{\lineartransforml{#1}{#2}{#3} + \biasparaml{#1}{#2}}
\newcommand{\weightparamletter}{\mymacro{\matvar{W}}}

\newcommand{\weightparaml}[2]{\mymacro{\weightparamletter^{(#2)}_{\mathrm{#1}}}}
\newcommand{\biasparamletter}{\mymacro{\vecvar{b}}}

\newcommand{\biasparaml}[2]{\mymacro{\biasparamletter^{(#2)}_{\mathrm{#1}}}}

\newcommand{\dropout}[1]{\textsc{Dropout}(#1)}

\newcommand{\thehiddenstateli}[2]{\mymacro{\vecvar{h}^{(#1)}_{#2}}}
\newcommand{\thehiddenstatedropoutli}[2]{\mymacro{\cancel{\vecvar{h}}^{(#1)}_{#2}}}
\newcommand{\initialhiddenstateparaml}[1]{\mymacro{\vectorparaml{w}{0}{#1}}}
\newcommand{\theinputgateli}[2]{\mymacro{\vecvar{i}^{(#1)}_{#2}}}
\newcommand{\theforgetgateli}[2]{\mymacro{\vecvar{f}^{(#1)}_{#2}}}
\newcommand{\thecandidateli}[2]{\mymacro{\vecvar{g}^{(#1)}_{#2}}}
\newcommand{\theoutputgateli}[2]{\mymacro{\vecvar{o}^{(#1)}_{#2}}}
\newcommand{\thememorycellli}[2]{\mymacro{\vecvar{c}^{(#1)}_{#2}}}
\newcommand{\thehiddeni}[1]{\mymacro{\vecvar{h}_{#1}}}
\newcommand{\theinputembeddingi}[1]{\mymacro{\vecvar{x}_{#1}}}
\newcommand{\theembeddingmatrix}{\mymacro{\matvar{E}}}

\allowdisplaybreaks

\title{Information Locality as an Inductive Bias for Neural Language Models}

\author{
 \textbf{Taiga Someya}$^{1}$\thanks{~~This research was conducted while visiting ETH Zürich.}~\;~
 \textbf{Anej Svete}$^{2}$ ~\;~
 \textbf{Brian DuSell}$^{2}$\\
 \textbf{Timothy J. O'Donnell}$^{3}$ ~\;~
 \textbf{Mario Giulianelli}$^{2}$ ~\;~
 \textbf{Ryan Cotterell}$^{2}$ \\
 \textsuperscript{1}The University of Tokyo ~\;~
 \textsuperscript{2}ETH Zürich ~\;~
 \textsuperscript{3}McGill University ~\;~
 \\
   \texttt{\href{mailto:taiga98-0809@g.ecc.u-tokyo.ac.jp}{taiga98-0809}}\texttt{@g.ecc.u-tokyo.ac.jp} \quad \texttt{\href{mailto:timothy.odonnell@mcgill.ca}{timothy.odonnell}}\texttt{@mcgill.ca}  \\
   \{\texttt{\href{mailto:asvete@inf.ethz.ch}{asvete}},
   \texttt{\href{mailto:brian.dusell@inf.ethz.ch}{brian.dusell}}, \texttt{\href{mailto:mario.giulianelli@inf.ethz.ch}{mario.giulianelli}}, \texttt{\href{mailto:ryan.cotterell@inf.ethz.ch}{ryan.cotterell}}\}\texttt{@inf.ethz.ch}
}

\begin{document}

\maketitle

\begin{abstract}
Inductive biases are inherent in every machine learning system, shaping how models generalize from finite data.
In the case of neural language models (LMs), debates persist as to whether these biases align with or diverge from human processing constraints.
To address this issue, we propose a quantitative framework that allows for controlled investigations into the nature of these biases.
Within our framework, we introduce \mlocal entropy---an information-theoretic measure derived from average lossy-context surprisal---that captures the local uncertainty %
of a language by quantifying how effectively the $m-1$ preceding symbols disambiguate the next symbol.
In experiments on both perturbed natural language corpora and languages defined by probabilistic finite-state automata (\pfsaAcr{}s), we show that languages with higher \mlocal entropy are more difficult for Transformer and LSTM LMs to learn.
These results suggest that neural LMs, much like humans, are highly sensitive to the local statistical structure of a language.
\begin{center}
    \faGithub~\url{https://github.com/rycolab/lm-inductive-bias}
\end{center}
\end{abstract}

\section{Introduction}
\begin{figure}[t]
    \centering
    \includegraphics[width=\linewidth]{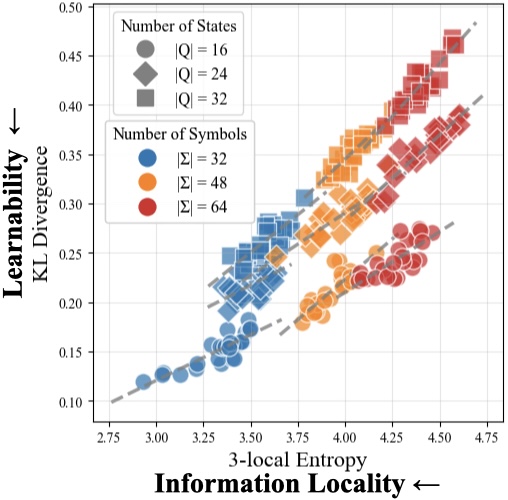}
    \caption{KL divergence (Transformer LM) as a function of the 3-local entropy of the language generated from a \pfsaAcr{} in Experiment 2. LMs perform better at languages with lower local entropy.}
    \label{fig:main_fig}
\end{figure}
Every machine learning system has some form of inductive bias: given a finite training sample with potentially infinitely many compatible hypotheses, its architecture and learning algorithm predispose it to favor certain generalizations over others \cite{mitchell1980need,Rawski2019NoFL}.
The concept of inductive bias is central to the growing discussion of whether the inductive biases of neural network LMs align with the cognitive pressures that shape human language learning.
In a 2023 New York Times article, \citeauthor{Chomsky_etal_2023} argued that neural LMs possess inductive biases fundamentally different from human cognitive constraints, a claim that has motivated theoretical rebuttals \citep{piantadosi_modern_2024} as well as empirical research to test the extent of this divergence \citep{kallini-etal-2024-mission,Ahuja2024Learning}.\looseness=-1

In particular, \citet{kallini-etal-2024-mission} demonstrated that perturbing natural language corpora to alter their sequential structure, making the languages less human-like, renders languages harder for neural LMs to learn.
While their findings suggest that disrupting local structure (e.g., through local shuffling transformations) impacts learnability, they do not isolate the specific property responsible for this effect.
To rigorously assess whether a neural LM's inductive biases align with human constraints, we must identify \emph{quantifiable properties} of language that affect human learning difficulty and \emph{systematically manipulate} these properties in controlled experiments with neural LMs.
Information-theoretic models of language processing, which view the structure of languages as shaped by language users' joint optimization of informativity and complexity, provide a promising framework for identifying these properties.\looseness=-1

In this paper, we turn to the principles of \defn{information locality}, which suggest that language is structured to minimize linear distance between linguistic elements with high mutual information~\citep{gibson1998locality,gibson2001dependency,futrell-2019-information,futrell-etal-2020-lossy}.
Information locality is thought to arise from the memory limitations of human processors, which make it challenging to integrate long-range linguistic dependencies \citep{hahn-etal-2022-resource}.
The influence of these cognitive constraints has been observed across multiple timescales, in both language comprehension and production, and across diverse languages of the world \citep{hahn_modeling_2021,Futrell2023InformationtheoreticPI,Futrell2024LinguisticSF}.
Demonstrating that a neural LM's learnability of a language is influenced by its local predictability would reveal an inductive bias that aligns with the functional pressures shaping human language processing.\looseness=-1

As a first step in this direction, we propose using \mlocal entropy, an information-theoretic measure designed to quantify a linear notion of local predictability which can be derived from the principles of information locality \citep{futrell-etal-2020-lossy}.\footnote{
   More specifically, \mlocal entropy can be derived from the lossy-context surprisal theory of language comprehension \citep{futrell-etal-2020-lossy}, a theory belonging to the expectation-based family of theories of language processing \citep{hale-2001-probabilistic,levy-2008}. See \cref{sec:mlocal_entropy} for details.
}
We study how \mlocal entropy affects learnability in two sets of experiments: one where we perturb natural language corpora, and another where we randomly generate probabilistic finite-state automata (\pfsaAcr{}s).
We train LSTM and Transformer LMs on these languages and examine whether neural LMs' difficulty in learning a language is predicted by the language's \mlocal entropy, to see if neural LMs and humans share an inductive bias for information locality.
Our experiments demonstrate that \mlocal entropy negatively correlates with the ability of a neural LM to learn a probabilistic language.
Specifically, our experiment with an English natural language corpus---and its perturbed variants---reveals that LSTM and Transformer LMs show systematic degradation in performance as \mlocal entropy increases, even when global and next-symbol entropy are held constant.
Furthermore, in experiments using \pfsaAcr{}s, we manipulate the properties of languages more \emph{systematically} and show that this trend is not an artifact of the corpora or particular perturbation functions used in our experiments.

\section{Formal Background}\label{sec:technical-preliminaries}

\subsection{Languages and Language Models}\label{sec:background}
An \defn{alphabet} $\alphabet$ is a finite, non-empty set of symbols.
The \defn{Kleene closure} $\klAlphabet$ is the set of all strings with symbols from $\alphabet$. We use $|\str|$ to denote the length of $\str \in \klAlphabet$.
A \defn{language} is a subset of $\klAlphabet$.

A \defn{language model} $\pLM$ is a probability distribution over $\klAlphabet$.
The \defn{prefix probability} $\pLMprefixFun{\str}$ is the probability that a string begins with $\str \in \klAlphabet$:
\begin{equation} \label{eq:prefix-prob}
\pLMprefixFun{\str} \defeq \sum_{\str' \in \klAlphabet} \pLMFun{\str\str'}
\end{equation}
The \defn{conditional prefix probability} of a string $\str' \in \klAlphabet$ given another string $\str$ is given by
\begin{equation}\label{eq:ratio}
\pLMprefixFun{\str' \mid \str} = \frac{\pLMprefixFun{\str\str'}}{\pLMprefixFun{\str}}.
\end{equation}

Using the notion of a conditional prefix probability, one can factorize a language model $\pLM$ as\footnote{Modern neural LMs (e.g., LSTMs, Transformers) directly parameterize the \emph{next–symbol} probability distribution $\qLMprefix(\sym \mid \str)$ over $\eosalphabet$, which coincides with the conditional prefix probability $\pLMprefixFun{\sym \mid \str}$. Training therefore maximizes the log‐likelihood of observed symbols, and the full string probability $\qLMFun{\str}$ is recovered via the product in~\eqref{eq:conditional-prob-product}.}
\begin{equation}
    \label{eq:conditional-prob-product}
    \pLMFun{\str} = \pLMprefixFun{\eos \mid \str} \prod_{t= 1}^{|\str|} \pLMprefixFun{\sym_t \mid \strlt},
\end{equation}
where each $\pLMprefix\left(\sym_t \mid \strlt\right)$ is a distribution over $\eosalphabet \defeq \alphabet \cup \{\eos\}$, where $\eos \not\in \alphabet$ is a distinguished \underline{e}nd-\underline{o}f-\underline{s}tring symbol, and
where we define
\begin{equation}
\label{eq:eos-prob}
    \pLMprefixFun{\eos \mid \str} \defeq \frac{\pLMFun{\str}}{\pLMprefixFun{\str}}.
\end{equation}
We define $\pLM$'s \defn{infix probability} $\pLMinfix$ as
\begin{equation} \label{eq:infix-prob}
\pLMinfixFun{\str} \defeq \sum_{\str' \in \klAlphabet} \sum_{\str'' \in \klAlphabet} \pLMFun{\str'\str\str''}.
\end{equation}
Note that, despite denoting the probability of an event, $\pLMprefix$ and $\pLMinfix$ are \emph{not} probability distributions over $\kleene{\alphabet}$.
In order to convert them into probability distributions, we have to renormalize them.
However, the the sums $\prefixnorm \defeq \sum_{\str \in \kleene{\alphabet}} \pLMprefixFun{\str}$ and $\infixnorm \defeq \sum_{\str \in \kleene{\alphabet}} \pLMinfixFun{\str}$ may diverge.
A necessary and sufficient condition for $\prefixnorm$ to be finite is that the expected length $\meanStrLen \defeq \E_{\str \sim \pLM} |\str|$ under $\pLM$ is finite.\footnote{See \citet[][Prop. 1]{opedal-etal-2024-role}.}
Thus, in the following, we assume that $\meanStrLen \defeq \E_{\str \sim \pLM} |\str| < +\infty$, which implies that we can \emph{normalize} the prefix probability function to form probability distributions.
We do not require that $\infixnorm$ be finite in this paper.
However, one can show that $\infixnorm$ is finite if and only if $\E_{\str \sim \pLM} |\str|^2 < +\infty$.

\paragraph{Random Variables.}
We will use the following random variables in our exposition.
First, let $\rvstr$ be a $\klAlphabet$-valued random variable distributed as
\begin{equation}
    \prob(\rvstr = \str) \defeq \pLMFun{\str}.
\end{equation}
Then, let $\rvPrefix$ be a $\klAlphabet$-valued random variable distributed according to
\begin{equation}
    \prob(\rvPrefix = \str) \defeq \frac{\pLMprefixFun{\str}}{\prefixnorm}.
\end{equation}
Next, let $\rvNextSymbol$ be a $\eosalphabet$-valued random variable jointly distributed with $\rvPrefix$ according to
\begin{equation}
    \prob(\rvNextSymbol = \sym \mid \rvPrefix = \str) \defeq \pLMprefixFun{\sym \mid \str}.
\end{equation}
Finally, let $\rvctx$ be a $\mAlphabet$-valued random variable distributed according to $\pLMinfix$ normalized over $\mAlphabet$:
\begin{equation} \label{def:normalized_ctx_probability}
    \prob\left(\rvctx = \ctx\right) \defeq \frac{\pLMinfixFun{\ctx}}{\sum_{\ctx' \in \mAlphabet} \pLMinfixFun{\ctx'}}.
\end{equation}
$\prob\left(\rvctx = \ctx\right)$ can be interpreted as observing $\ctx$ as a length-$(m - 1)$ substring of a string sampled from $\pLM$.
Additionally, the random variable $\rvContextNextSymbol$ is conditionally distributed, given $\rvctx$, according to
\begin{equation}
    \prob(\rvContextNextSymbol = \sym \mid \rvctx = \ctx) = \sum_{\str' \in \kleene{\alphabet}} \pLMprefixFun{\sym \mid \str' \ctx} \pLMprefixFun{\str'},
\end{equation}
i.e., as the next symbol given that the previous $m - 1$ symbols were $\ctx$.

\subsection{Global Entropy and \Mlocal Entropy} \label{sec:global_and_local_entropy}
The concept of \emph{entropy}, as introduced by \citet{shannon1948}, provides a foundational framework for quantifying uncertainty in language.
Depending on how one defines the underlying probability distribution over linguistic units, entropy can capture different aspects of language complexity.
In this paper, we discuss two versions of entropy---\emph{global} entropy (i.e., the \citeauthor{shannon1948} entropy) and \emph{\mlocal} entropy---each capturing different facets of language complexity.

\subsubsection{Global Entropy} \label{sec:global_entropy}
The \defn{global entropy} of $\pLM$ is defined as
\begin{equation}\label{eq:entropy}
    \entropyFun{\rvstr} \defeq - \sum_{\str \in \klAlphabet} \pLMFun{\str} \log\pLMFun{\str}.
\end{equation}
Note that \Cref{eq:entropy} may be infinite for some $\pLM$.
However, for the remainder of our paper, we will assume that $\entropyFun{\rvstr} < + \infty$.
\Cref{eq:entropy} reflects the uncertainty in the distribution over all possible strings $\str \in \klAlphabet$: Higher global entropy indicates that $\pLM$ distributes probability mass more uniformly across strings.\footnote{Because $\kleene{\alphabet}$ is countably infinite, there is no uniform distribution over $\kleene{\alphabet}$.}

We further define the (global) next-symbol entropy for finite-mean-length language models as the weighted average of the entropy of all local next-symbol distributions, averaging over all possible contexts $\str \in \klAlphabet$, and weighted by the normalized prefix probability of $\str$.\footnote{The weighting cannot be uniform, as $\klAlphabet$ is infinite.}
First, for a specific context $\str$, we define the context-specific next-symbol entropy as follows:
\begin{equation}
\begin{aligned}\label{eq:context-specific}
    &\entropyFun{\rvNextSymbol \mid \rvPrefix = \str} \\
    &\quad \qquad \defeq -\sum_{\sym \in \eosalphabet} \pLMprefixFun{\sym \mid \str} \log \pLMprefixFun{\sym \mid \str}.
\end{aligned}
\end{equation}
Then, using \Cref{eq:context-specific}, we construct the \defn{next-symbol entropy} as
\begin{subequations} \label{eq:next-symbol-entropy}
    \begin{align}
          \mathrm{H}(&\rvNextSymbol \mid \rvPrefix) \nonumber \\
          &= \sum_{\str \in \klAlphabet} \prob(\rvPrefix = \str)\,\entropyFun{\rvNextSymbol \mid \rvPrefix = \str} \\
          &= \sum_{\str \in \klAlphabet} \frac{\pLMprefixFun{\str}}{\prefixnorm}  \,\entropyFun{\rvNextSymbol \mid \rvPrefix = \str} \\
        &= \frac{1}{\meanStrLen + 1} \sum_{\str \in \klAlphabet} \pLMprefixFun{\str}   \,\entropyFun{\rvNextSymbol \mid \rvPrefix = \str} \\
        &= \frac{1}{\meanStrLen + 1} \,\entropyFun{\rvstr}. \label{eq:next-sym-entropy-length}
    \end{align}
\end{subequations}
where we exploit the identity $\meanStrLen + 1 = \prefixnorm$ and where the last equality follows from \citet[][Thm. 2.2]{malagutti-etal-2024-role}.
What \Cref{eq:next-sym-entropy-length} tells us is that the next-symbol entropy is \emph{proportional} to the global entropy.
So, if both were used as predictors in a linear model, they would yield identical predictive power.\looseness=-1

\paragraph{Invariance of global and next-symbol entropy.} \label{sec:invariance_of_global_entropy}
Our goal in this paper is to isolate how local uncertainty of a language affects the performance of neural LMs in learning the language.
Global entropy and its length-normalized variant, next-symbol entropy, are obvious benchmarks, yet neither is sensitive to the local statistical structure we want to manipulate.
Global entropy is invariant to bijective transformations of $\klAlphabet$; that is, for any bijection $\bijection \colon \klAlphabet \rightarrow \klAlphabet$, $\entropyFun{\rvstr} = \entropyFun{\bijection(\rvstr)}$.
Next-symbol entropy is only slightly more rigid: it is, in general, not conserved under bijection, but it remains constant under any length-preserving bijection.
To capture the aspect of local uncertainty that these measures miss, we introduce \mlocal entropy in the next section.
Unlike global entropy and next-symbol entropy, \mlocal entropy \textit{does} change under the length-preserving, bijective perturbations used in our experiment (\Cref{sec:constructing_langs}), which makes it possible to create a set of counterfactual corpora where each corpus has the same global and next-symbol entropy but has different \mlocal entropy.

\subsubsection{\Mlocal Entropy} \label{sec:mlocal_entropy}
Next-symbol entropy measures uncertainty over next-symbol predictions conditioned on the full available context, averaged across all possible contexts.
This can be seen as the limit of a \emph{local} quantitification of uncertainty, which captures the unpredictability of the next symbol given a fixed amount of preceding context.
We term this fixed-context uncertainty measure \defn{\mlocal local entropy}.

Given any $\ctx \in \mAlphabet$, we can compute
\begin{align}\label{def:mlocal_entropy_given_ctx}
    &\entropyFun{\rvContextNextSymbol \mid \rvctx=\ctx}  = -\sum_{\sym \in \eosalphabet}
        \prob(\rvContextNextSymbol=\sym \mid \rvctx=\ctx) \nonumber \\
    &\qquad\qquad\qquad \log \prob(\rvContextNextSymbol=\sym \mid \rvctx=\ctx).
\end{align}
This captures the unpredictability of a symbol $\symy$ after observing a given local context $\ctx$.
We can then say that the \defn{\mlocal entropy} of $\pLM$ is an expectation over possible contexts $\ctx \in \mAlphabet$, with each context weighted by $\prob\left(\rvctx = \ctx\right)$:
\begin{equation}
    \label{def:mlocal_entropy}
    \entropyFun{\rvContextNextSymbol \mid \rvctx} =
    \sum_{\ctx \in \mAlphabet} \prob\left(\rvctx = \ctx\right) \, \entropyFun{\rvContextNextSymbol \mid \rvctx = \ctx}.
\end{equation}
This yields a measure of local complexity that can differ from global entropy by more than a multiplicative constant.
Even when two languages have identical \emph{global} entropy, their \emph{\mlocal} entropies reflect differences in how reliably the local context predicts the next symbol.
Importantly, unlike global entropy, local entropy is \emph{not} necessarily preserved under bijective transformations of $\klAlphabet$, which enables us to assess the impact of such transformations on learnability.
Our experiments show that transformations that disrupt local statistical structure are associated with how well neural LMs are able to learn a probabilistic language.

\paragraph{\Mlocal entropy and lossy-context surprisal.}
As a generalization of the surprisal model of language processing difficulty~\citep{hale-2001-probabilistic,levy-2008}, \citet{futrell-etal-2020-lossy}
propose \emph{lossy‐context surprisal}.
In this framework, the predicted difficulty of an upcoming word is proportional to the word’s expected log probability given a \emph{lossy} memory representation of the preceding context.
A memory encoding function specifies a distribution over such representations.
If that function keeps only the $m\!-\!1$ symbols immediately preceding the target word, the metric collapses to the familiar $m$-gram surprisal.
Averaging this quantity over all possible contexts with language-specific weights yields the average (lossy-context) surprisal of a language \citep{futrell-2019-information,hahn_modeling_2021}.
When those weights are the contexts' infix probabilities (see~\cref{eq:infix-prob}), the average coincides with our definition of \mlocal entropy (\cref{def:mlocal_entropy}).
To our knowledge, no prior work has linked lossy-context surprisal directly to language model learnability---a connection that our work aims to explore.

\subsection{Probabilistic Finite-State Automata} \label{sec:pfsa}
\begin{definition}
    A \defn{probabilistic finite-state automaton} (\pfsaAcr{}) is a 5-tuple $\wfsatuple$ where
    \begin{itemize}[nosep,noitemsep]
        \item $\alphabet$ is an alphabet,
        \item  $\states$ is a finite set of states,
        \item $\trans \subseteq \states \times \alphabet \times \left[0, 1\right] \times \states$ is a finite set of weighted transitions, rendered as $\edge{\stateq}{\sym}{w}{\stateq'}$ with $\sym \in \alphabet$ and $w \in [0,1]$,
        \item $\initf, \finalf\colon \states \rightarrow \left[0, 1\right]$ are the initial and final weighting functions,
        \item $\initf$ satisfies $\sum_{\stateq \in \states} \initf\left(\stateq\right) = 1$, and
        \item for all $\stateq \in \states$, $\sum_{\edge{\stateq}{\sym}{w}{\stateq'} \in \trans} w + \finalf\left(\stateq\right) = 1$.
    \end{itemize}
\end{definition}
A \defn{path} $\apath$ in a \pfsaAcr $\wfsa$ is a sequence of consecutive transitions
$\edge{\stateq_0}{\sym_1}{w_1}{} \cdots \edge{}{\sym_{\pathlen}}{w_{\pathlen}}{\stateq_{\pathlen}}$.
We define its \defn{scan} as $\yield\left(\apath\right) \defeq \sym_1 \cdots \sym_{\pathlen}$.
$\paths(\automaton, \str)$ denotes the set of all paths in $\automaton$ that scan $\str \in \klAlphabet$.
The \defn{inner path weight} of $\apath$ is $\innerpathweightFun{\apath} = \prod_{\idx = 1}^\pathlen w_\idx$, and its \defn{path weight} is $\pathweightFun{\apath} = \initfFun{\stateq_0}\innerpathweightFun{\apath}\finalfFun{\stateq_{\pathlen}}$.

A \pfsaAcr{} $\wfsa$ induces a language model $\pLMwfsa$ as
\begin{equation}
    \pLMwfsaFun{\str} \defeq \sum_{\apath \in \paths\left( \automaton, \str \right) } \pathweightFun{\apath}.
\end{equation}

Studying \pfsaAcr{}s not only allows us to perform controlled experiments but also enables us to compute many quantities of interest exactly.
\cref{sec:pfsa_detail} contains a collection of closed-form solutions for computing various quantities of interest, including the string (prefix and infix) probabilities and the \mlocal entropy of the induced language model.

\section{Experiment 1: LM Performance along the \Mlocal Entropy Continuum}
\label{sec:exp_natural}
In the first experiment, we investigate the relationship between local entropy and LM performance using a natural language corpus.
We hypothesize that \emph{local entropy} is a key factor in determining how easily an LM learns a language.
To test this hypothesis, we apply a length-preserving, bijective perturbation function to a natural language corpus that alters its local structure.
As we will see in the following section, this results in a counterfactual perturbed corpus~\citep[cf.][]{kallini-etal-2024-mission}, where language models have different \mlocal entropy but the same global entropy and next-symbol entropy.
We then train neural LMs on the naturally occurring corpus and the perturbed one and study how local entropy affects the neural LMs' performance.\looseness=-1

\subsection{Constructing Languages with Different Local Complexity}
\label{sec:constructing_langs}
Here, we detail several specific transformations implemented in our experiments, refining the perturbation functions of \citet{kallini-etal-2024-mission}.
Note that all of the perturbation functions defined below are both bijective and length-preserving.

\paragraph{\textsc{DeterministicShuffle}.}
Given a string of length $T$, $\str = \sym_1\sym_2\cdots\sym_T$, the function
applies a length-specific permutation $\sigma_T$ of the positions
$\{1,\ldots,T\}$.
Each $\sigma_T$ is sampled \emph{once} at the start of the experiment (using
a fixed pseudorandom seed) and then reused.
Therefore, two strings of the same length are shuffled in exactly the same way, whereas strings of different lengths are transformed according to different permutations.
By construction, every $\sigma_T$ is a bijection on $\{1,\ldots,T\}$, so the string length is left unchanged.

\paragraph{\textsc{Reverse}.}
This function reverses the entire sequence of symbols.
Formally, given a string \(\str = \sym_1\,\sym_2 \cdots\symT\), the \textsc{Reverse} mapping produces \(\sym_T\,\sym_{T-1}\cdots\sym_1\).
It is trivially invertible (by applying the same operation again), making it a bijection.\looseness-1

\paragraph{\textsc{EvenOddShuffle/OddEvenShuffle}.}
Let $\str = \sym_1\sym_2\cdots \symT$ be a string of length $T$.
Define two subsequences $O(\str) = \sym_1\sym_3\sym_5\cdots$ and $E(\str) = \sym_2\sym_4\sym_6\cdots$ that collect all symbols in $\str$ at \emph{even} positions and \emph{odd} positions, respectively.
We then define $\textsc{EvenOddShuffle}(\str) = E(\str)O(\str)$ and $\textsc{OddEvenShuffle}(\str) = O(\str)E(\str)$.

\paragraph{\textsc{K-localDeterministicShuffle}.}
Let $\str = \sym_1 \sym_2 \cdots \symT$ be a string in $\klAlphabet$, which we partition into consecutive windows of size $k$.
For the $i^{\text{th}}$ window, $\sym_{(i-1)k+1} \cdots \sym_{ik}$, we apply a fixed permutation $\pi^{k}_i$ determined by window size $k$, window index $i$ and a global random seed.
Formally, the \textsc{K-localDeterministicShuffle} of $\str$  produces
$
    \bigl(\,
    \pi^{k}_1(\sym_1 \cdots \sym_k),\;
    \pi^{k}_2(\sym_{k+1} \cdots \sym_{2k}),\;\ldots
    \bigr).
$\footnote{If the string length $T$ is not a multiple of $k$, then the final window, which contains $\ell\,(< k)$ symbols, is permuted by applying the length-$l$ permutation $\pi^{\ell}_{i}$ to all the available symbols in that window.
}

\begin{table*}[t]
    \centering
    \begin{tabular}{ccccc}
    \toprule
    & \multicolumn{2}{c}{50K samples} & \multicolumn{2}{c}{200K samples} \\
        \cmidrule(lr){2-3} \cmidrule(lr){4-5}
        $m$ & MAE & MRE $(\%)$ & MAE & MRE $(\%)$ \\
        \midrule
        2 & 0.0015 $\pm$ 0.0009 & 0.04 $\pm$ 0.02 & 0.0007 $\pm$ 0.0005 & 0.02 $\pm$ 0.01 \\
        3 & 0.0132 $\pm$ 0.0058 & 0.36 $\pm$ 0.14 & 0.0046 $\pm$ 0.0021 & 0.13 $\pm$ 0.05 \\
        4 & 0.1267 $\pm$ 0.0615 & 3.63 $\pm$ 1.53 & 0.0522 $\pm$ 0.0267 & 1.49 $\pm$ 0.67 \\
        5 & 0.4702 $\pm$ 0.1953 & 13.82 $\pm$ 4.78 & 0.2599 $\pm$ 0.1229 & 7.61 $\pm$ 3.10 \\
        \bottomrule
    \end{tabular}
    \caption{Mean absolute error (MAE) and mean relative error (MRE) of the $m$‑gram estimator of \mlocal entropy (averaged over 16 \pfsaAcr{}s).}
    \label{tab:ngram-m-local-error}
  \vspace{-12pt}
\end{table*}

\subsection{Estimating the \Mlocal Entropy} \label{sec:estimating_mlocal_entropy}
Unfortunately, when we only observe a corpus, the \mlocal entropy of the data-generating distribution is not known.
In this experiment, we estimate it using an \ngram language model implemented with KenLM~\cite{heafield-2011-kenlm}.
Given a corpus $\dataset$, we train an \ngram model on $\dataset$ to get the estimated conditional probability distribution $\pLMhatFun{\sym \mid \ctx}$ for $\ctx \in \mAlphabet$.
Plugging this estimated probability distribution into \Cref{def:mlocal_entropy_given_ctx}, we can compute\looseness-1
\begin{equation} \label{eq:estimated_mlocal_entropy_given_ctx}
   \!\! \entropyhatFun{\rvContextNextSymbol \mid \rvctx = \ctx} = - \frac{1}{\countsFun{\ctx}}\sum_{\sym \in \dataset} \log \pLMhat\left(\sym \mid \ctx \right),
\end{equation}
where $\countsFun{\ctx}$ is the number of times $\ctx$ appears in~$\dataset$.
The normalized infix probability is estimated as
\begin{equation} \label{eq:estimated_normalized_ctx_probability}
    \pLMhatFun{\rvctx = \ctx} = \frac{\countsFun{\ctx}}{N_\mathrm{total}},
\end{equation}
where $N_\mathrm{total} =\sum_{\ctx' \in \mAlphabet} \countsFun{\ctx'}$.

Given these and \Cref{def:mlocal_entropy}, we can approximate the \mlocal entropy as
\begin{subequations} \label{def:mlocal_entropy-est}
    \begin{align}
        \mathrm{H}(\rvContextNextSymbol& \mid \rvctx) \notag \\
        &= \sum_{\ctx \in \mAlphabet} \pLMhatFun{\rvctx = \ctx} \, \entropyhatFun{\rvContextNextSymbol \mid \rvctx = \ctx} \\
        &= -\frac{1}{N_\mathrm{total}} \sum_{\ctx\sym \in \dataset} \log \pLMhatFun{\sym \mid \ctx}.
    \end{align}
\end{subequations}
This estimator is a practical proxy for the quantity in \Cref{def:mlocal_entropy}.
The \mlocal entropy of each corpus is estimated by an \ngram model with order $m-1$ trained on the concatenation of the training, validation, and test set of the corpus.

\subsubsection{Validating \Mlocal Entropy Estimation} \label{sec:ngram_validation}
To ensure the reliability of our \mlocal entropy estimation with \ngram models, we quantify the discrepancy between the \mlocal entropy estimated with an empirical \ngram model with order $m-1$ and the true \mlocal entropy that can be calculated analytically for \pfsaAcr{}s.
We sample 16 \pfsaAcr{}s covering two values for the number of states ($\nstates\in \{8, 16\}$), two alphabet sizes ($\nsymbols \in \{32, 48\}$), two different random topologies, and two different random transition weights (see \Cref{sec:pfsa-generation} for the details).
For each \pfsaAcr{}, we generate two corpora, with 50K and 200K strings respectively, and then estimate \mlocal entropy using \ngram models with order $m-1$ for $m=2,\dots,5$.
For every $(m,\text{PFSA})$ combination, we compute the absolute (AE) and relative (RE) error between the estimated and the true $m$‑local entropy, then average the errors across the 16 \pfsaAcr{}s.
\Cref{tab:ngram-m-local-error} summarizes our results.
With 200K strings, the estimator achieves relative error below one percent up to $m=4$ and remains under $8\%$ even for $m=5$, the most challenging setting.
Increasing the corpus size from 50K to 200K consistently reduces both absolute and relative error, confirming that additional data yield further improvements.
In sum, the $m$‑gram estimator provides a faithful approximation to the true $m$‑local entropy for all $m$ studied.\looseness=-1

\begin{table*}[t]
    \centering
    \begin{tabular}{ccccc}
    \toprule
     & 2-local entropy & 3-local entropy & 4-local entropy & 5-local entropy \\
    \midrule
    \textsc{Base}                & 6.67      & 4.27      & 2.92      & 2.45      \\
    \midrule
    \textsc{Reverse}             & 6.98      & 4.39      & 2.98      & 2.51      \\
    \textsc{EvenOddShuffle}      & 7.91      & 5.10      & 3.76      & 3.43      \\
    \textsc{OddEvenShuffle}      & 7.87      & 5.07      & 3.74      & 3.41      \\
    \textsc{LocalShuffle (k=3)} & 8.12 ± 0.08 & 5.05 ± 0.06 & 3.68 ± 0.07 & 3.39 ± 0.07 \\
    \textsc{LocalShuffle (k=4)} & 8.25 ± 0.07 & 5.17 ± 0.04 & 3.80 ± 0.04 & 3.56 ± 0.05 \\
    \textsc{LocalShuffle (k=5)} & 8.33 ± 0.07 & 5.25 ± 0.04 & 3.88 ± 0.04 & 3.64 ± 0.05 \\
    \textsc{LocalShuffle (k=6)} & 8.43 ± 0.07 & 5.31 ± 0.05 & 3.97 ± 0.06 & 3.72 ± 0.06 \\
    \textsc{LocalShuffle (k=7)} & 8.47 ± 0.07 & 5.36 ± 0.05 & 4.03 ± 0.06 & 3.78 ± 0.07 \\
    \textsc{DeterministicShuffle} & 8.77      & 5.73      & 4.60      & 4.41      \\
    \bottomrule
    \end{tabular}
    \caption{\Mlocal entropy values for the \textsc{Base} (original) corpus and the different perturbed corpora. \textsc{Local shuffle} refers to the \textsc{K-localDeterministicShuffle}. Values are shown as mean ± standard deviation (averaged over different random seeds).}
    \label{tab:local_entropy_bllip}
  \vspace{-5pt}
\end{table*}

\subsection{Experimental Setup}
\subsubsection{Neural Language Models}
\label{sec:neural_lms}

We investigate how varying \mlocal entropy in a language model impacts the performance of two widely used neural LM architectures: the LSTM~\citep{hochreiter-schmidhuber-1997-lstm} and the Transformer~\citep{vaswani-etal-2017-attention}.
We use a single-layer LSTM with 512-dimensional hidden units and a 4-layer causally-masked Transformer encoder with 768-dimensional representations, 3072-dimensional feedforward layers, and 12 attention heads.
Both are implemented in PyTorch~\cite{paszke-etal-2019-pytorch}.
Both architectures are trained on the training set via the standard language modeling objective across 5 random training seeds.
See \cref{sec:architectures,sec:experiments-details} for more details.

\subsubsection{Dataset}
We conduct our experiments on a subset of the Brown Laboratory for Linguistic Information Processing 1987--89 Corpus Release~1~\citep[BLLIP;][]{charniakeugene2000BLLIP}.\footnote{
    We also conduct the same set of experiments using the BabyLM corpus~\cite{choshen2024callpapers2ndbabylm}; results in \Cref{sec:appendix-babylm}.
}
Specifically, we adopt the same training, development, and test splits as \textsc{BLLIP-sm} in \citet{Hu_et_al_2020}, which comprise roughly 200K sentences, totaling around 5M tokens.
Starting from this original corpus, we apply the perturbation functions in \Cref{sec:constructing_langs} to produce perturbed corpora.
For both \textsc{DeterministicShuffle} and \textsc{K-localDeterministicShuffle}, we use 20 random seeds.
In the case of \mbox{\textsc{K-localDeterministicShuffle}}, we vary the parameters $k$ over the set $\{3,4,5,6,7\}$, yielding 20 perturbed corpora for each $k$.
This produces a total of 124 distinct corpora, including other perturbed corpora and the \textsc{Base} (original) corpus.

\subsubsection{Evaluating Learning Difficulty}
\label{sec:learning_language}
In this experiment, we rely on next-symbol cross-entropy as a measure of how well a trained language model $\qLM$ approximates the target distribution.
If $\pLM$ is the ground-truth language model, and $\qLM$ is any learned neural LM, we can estimate the next-symbol cross-entropy as:
\newcommand{\estimatedNextSymbolCrossEntropy}{\entropyhat_{\qLM}}
\begin{align} \label{eq:emp_ce}
  &\estimatedNextSymbolCrossEntropy(\rvNextSymbol \mid \rvPrefix) \notag \\
  &=-\,\frac{1}{\totalsymcount}\sum_{\str \in \dataset} \Bigl[\;\log\qLMprefix\bigl(\rvContextNextSymbol = \eos \,\bigl\vert\, \rvPrefix = \str\bigr)   \notag\\
  & \qquad + \sum_{t=1}^{|\str|}\log\qLMprefix\bigl(\rvContextNextSymbol = \sym_t \,\bigl\vert\, \rvPrefix = \strlt\bigr)\Bigr] ,
\end{align}
where $\dataset = \{\str^{(n)}\}_{n=1}^\datasetSize$ is a set of i.i.d.\ draws from $\pLM$, and $\totalsymcount = \sum_{\str \in \dataset} |\str| + 1$.
In this experiment, we evaluate each LM using the estimated next-symbol cross-entropy on the test set.

When comparing the ability of a neural LM to learn two different target language models, it is essential to account for the inherent entropy of each target language model. In the statistical setting, learning a language means estimating its probability distribution as closely as possible. Consequently, cross-entropy without taking into account the entropy can lead to incorrect conclusions; see, for example, \textsc{NondeterministicShuffle} in \citealp{kallini-etal-2024-mission}.
To address this, our experiments ensure that all corpora have the same inherent (global) entropy; see \Cref{sec:constructing_langs}, which allows us to safely compare cross-entropy results across different languages.

\begin{figure*}
    \centering
    \includegraphics[scale=0.24]{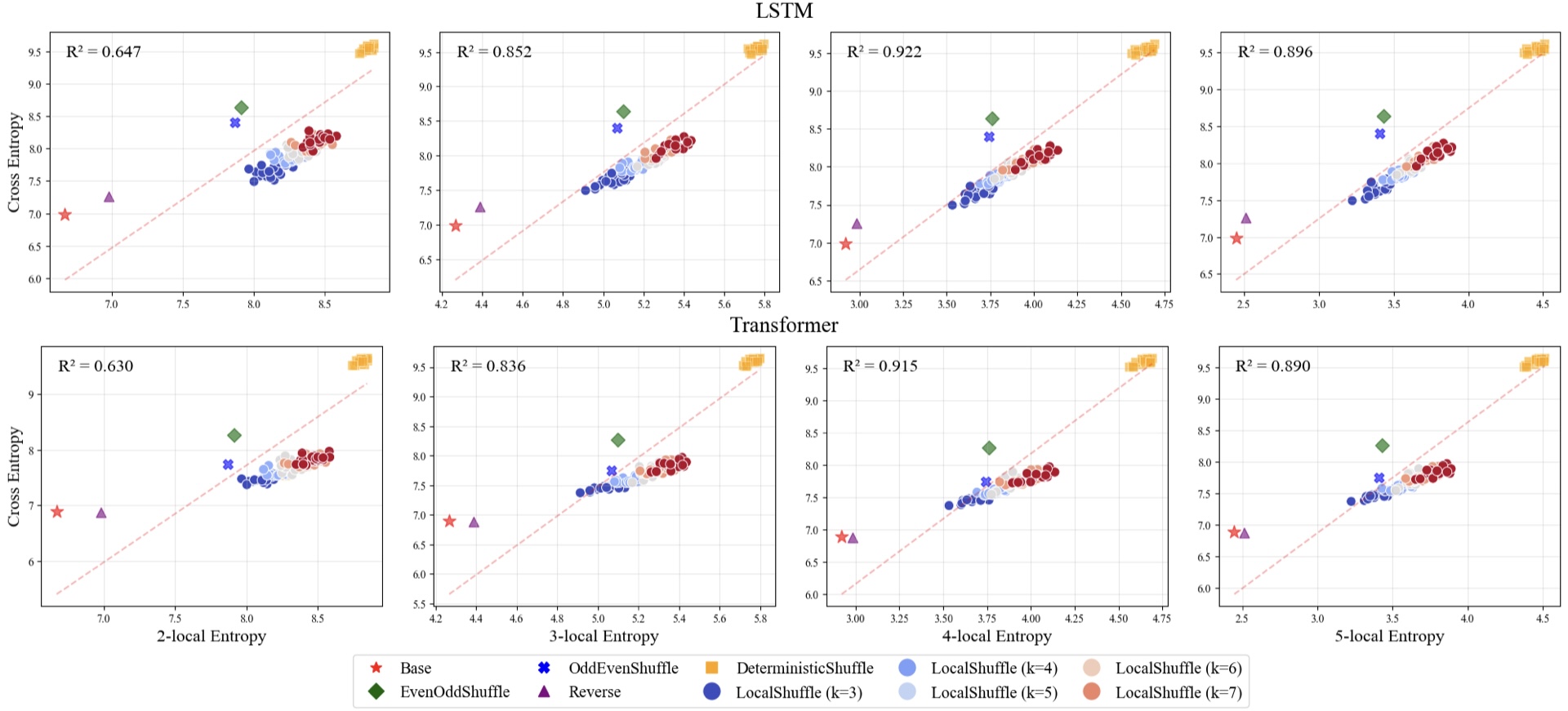}
    \caption{Scatter plots of next-symbol cross-entropy ($y$-axis) versus \mlocal entropy ($x$-axis) for $m \in \{2,3,4,5\}$, for both LSTM (top row) and Transformer LM (bottom row).
    Each marker type/color corresponds to a different perturbation (e.g., \texttt{Reverse}, \texttt{DeterministicShuffle}, \textsc{K-localDeterministicShuffle} with various window sizes, etc.).
    The red star indicates the unperturbed \texttt{Base} condition (original corpus).
    The dashed line in each panel is a linear fit, with $R^2$ indicating the coefficient of determination.}
    \label{fig:lm_performance_bllip}
  \vspace{-5pt}
\end{figure*}

\subsection{Results}
\paragraph{How do different perturbations affect \mlocal entropy?} \label{sec:shuffle_mlocal}

\Cref{tab:local_entropy_bllip} reports the \mlocal entropy values ($m \in \{2,3,4,5\}$) for the \textsc{Base} corpus and the various perturbed corpora.
\textsc{Reverse} barely changes the \mlocal entropy, whereas \textsc{EvenOddShuffle} and \textsc{OddEvenShuffle} increase it somewhat more.
In contrast, \textsc{K-localDeterministicShuffle} yields progressively higher entropy as the window size \(k\) grows, indicating a greater disruption of local ordering.
Finally, \textsc{DeterministicShuffle} produces the highest \mlocal entropies among all transformations.
Our results confirm that the bijective transformations we defined in \Cref{sec:constructing_langs} effectively generate new language models with different \emph{\mlocal} entropies from the original one, while preserving the \emph{global} entropy by design.
This yields a continuum of languages along a specific measurable axis of complexity rather than a qualitative notion of possibility as in \citet{kallini-etal-2024-mission}.\looseness-1

\paragraph{\!\Mlocal Entropy and LM Performance.} \label{sec:mlocal_lm_perf}
\Cref{fig:lm_performance_bllip} shows the relationship between the  \mlocal entropy (estimated by $m$-gram models; \Cref{sec:estimating_mlocal_entropy}) and the next-symbol cross-entropy of each neural LM on the test set.
We observe a strong positive correlation between \mlocal entropy and next-symbol cross-entropy for both neural architectures.
For example, with $m=4$, the coefficient of determination $R^2$ reaches 0.922 for the LSTM LM and 0.915 for the Transformer LM, indicating that higher local ambiguity (as measured by \mlocal entropy) generally leads to decreased performance (i.e., higher next-symbol cross-entropy) under both models.
Furthermore, since our transformations are designed to preserve global entropy and global next-symbol entropy, these results highlight the crucial role of \emph{local} entropy in the learnability of a language by neural LMs. This suggests that neural LMs inherently possess an inductive bias toward languages with lower local entropy.\looseness-1
\section{Experiment 2: Controlled Learnability Tests with \pfsaAcr{}s} \label{sec:exp_pfsa}

\begin{figure*}[t]
    \centering
    \includegraphics[scale=0.23]{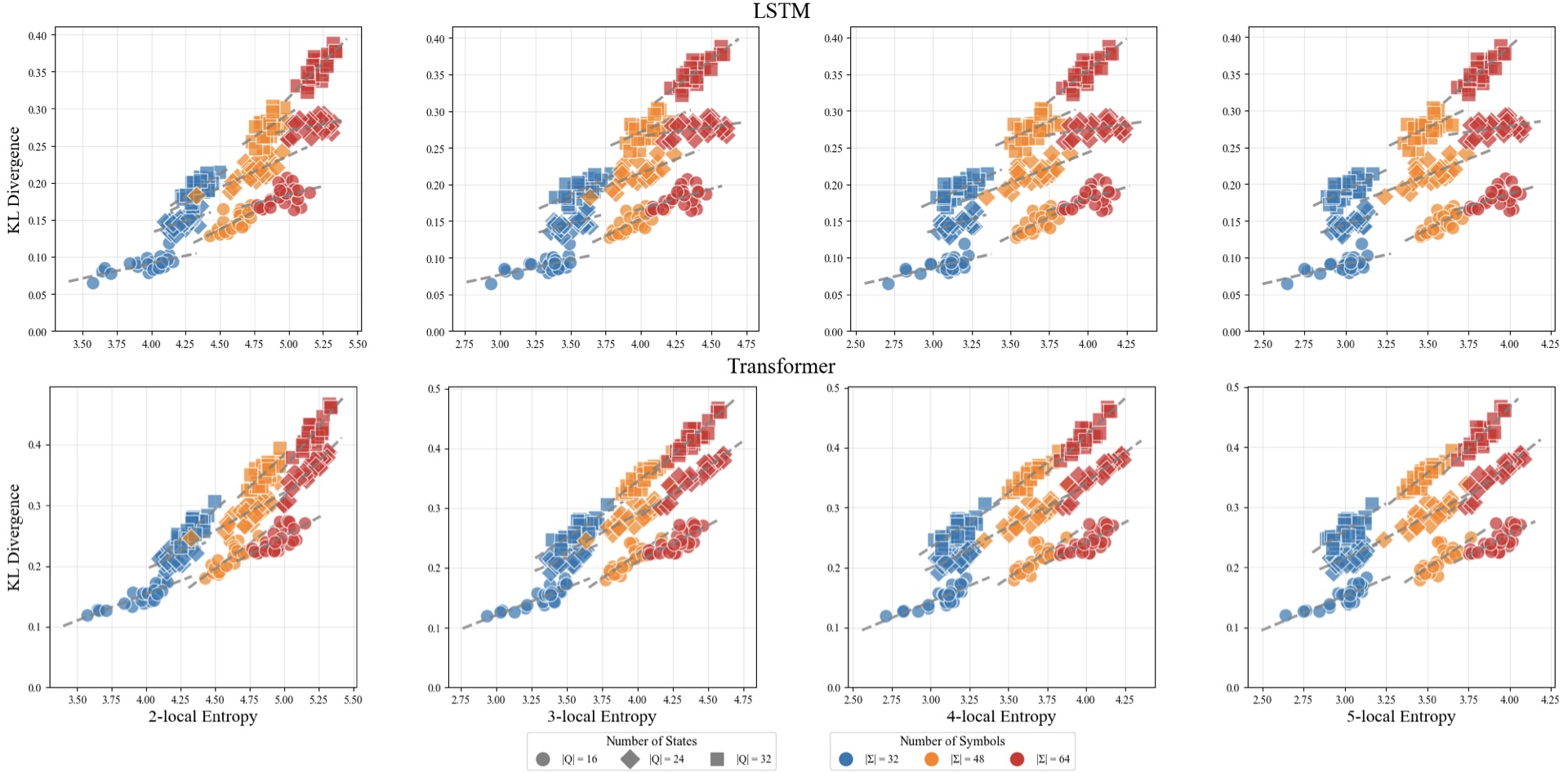}
    \caption{Scatter plots of symbol-level KL divergence (y-axis) versus \mlocal entropy (x-axis) for $m \in \{2,3,4,5\}$, for both LSTM (top row) and causally-masked Transformer encoder (Transformer; bottom row) models.
    Each marker type/color corresponds to a different combination of number of states ($\nstates$) and symbols ($\nsymbols$).
    The dashed line is a linear fit for each cluster.}
    \label{fig:lm_performance_pfsa}
   \vspace{-5pt}
\end{figure*}

Experiment 1 only focused on a specific English corpus and a specific set of perturbation functions.
To confirm that the results are not just an artifact of this experimental design but a fundamental property of neural LMs, we conduct a controlled experiment using \pfsaAcr{}s.
This also enables us to compute quantities of interest exactly, especially the \mlocal entropy of the induced language model.

\subsection{Experimental Setup}
We use the same neural LMs and training configurations as in \Cref{sec:exp_natural}, but we generate datasets using PFSAs (\Cref{sec:exp2-generating-datasets}) and evaluate neural LMs while controlling for global entropy (\Cref{sec:exp2-evaluating}).\footnote{Recall that in Experiment 1 it was unnecessary to control for global entropy.}

\subsubsection{Generating Datasets Using \pfsaAcr{}s}
\label{sec:exp2-generating-datasets}
We construct random \pfsaAcr{}s with alphabet sizes $\nsymbols \in \{32, 48, 64\}$ and numbers of states $\nstates \in \{16, 24, 32\}$.
For each of the nine configurations, we randomly generate 25 automata.
We do this by generating five random \pfsaAcr{} topologies (the underlying multi-graph) and five random weightings for each.
See \Cref{alg:random_dpfsa} in \cref{sec:pfsa-generation} for details.
We sample 20K strings for the training set, 5K for the validation set, and 5K for the test set from $\pLMwfsa$ for each \pfsaAcr{} $\wfsa$.

\subsubsection{Evaluating Learning Difficulty}
\label{sec:exp2-evaluating}
Using PFSAs allows us to compute a range of entropy-related values, including the inherent next-symbol entropy (\Cref{sec:pfsa}), which enables us to evaluate LMs based on KL divergence $\KL$.
Specifically, the estimated  $\KLhat$ is given by subtracting the next-symbol entropy of the PFSA from the estimated next-symbol cross-entropy of the LM (\Cref{eq:emp_ce}):
$
    \KLhat = \estimatedNextSymbolCrossEntropy(\rvNextSymbol \mid \rvPrefix) - \entropyFun{\rvNextSymbol \mid \rvPrefix}
$.
In this second experiment, we evaluate each LM using $\KLhat$ on the test set.\looseness=-1

\subsection{\Mlocal Entropy and LM Performance}
\Cref{fig:lm_performance_pfsa} shows the relationship between the \mlocal entropy of \pfsaAcr{}s (calculated analytically; \Cref{sec:pfsa_detail}) and the KL divergence of each neural LM on the test set; see \Cref{tab:mlocal_entropy_pfsa_correlation} for Pearson correlation coefficients.
The experimental results reveal a clear positive correlation between \mlocal entropy and $\KLhat$ across both architectures and all values of \(m = 2,3,4,5\), indicating that neural LMs find it more challenging to model distributions with higher local uncertainty.
The Transformer LM consistently shows higher $\KLhat$ compared to the LSTM within each topological cluster, suggesting that LSTMs are more effective at modeling these particular probability distributions~\citep{weiss-etal-2018-practical,borenstein-etal-2024-easy}.
Additionally, when $\nsymbols$ is constant, $\KLhat$ is higher for \pfsaAcr{}s with larger $\nstates$, consistent with \citet{borenstein-etal-2024-easy}.\looseness-1

\section{Discussion and Conclusion}
By proposing \mlocal entropy as a predictor of learning difficulty grounded in lossy-context surprisal theory and information locality principles, we provide a formal information-theoretic perspective that connects the inductive biases of LMs and the statistical properties of language thought to be shaped by functional pressures in humans \cite{gibson2001dependency,futrell-etal-2020-lossy}.
Through two sets of experiments---one on perturbations of a natural language corpus and another using \pfsaAcr{}s for the controlled generation of synthetic languages---we consistently find that both LSTM and Transformer architectures model languages with lower \mlocal entropy more effectively.
The shared sensitivity to information locality between artificial and human learners suggests a common inductive bias shaping both systems, possibly because both systems process language incrementally.\looseness-1

Our findings open several promising directions for future research. One avenue is to explore inductive biases beyond information locality, such as sensitivity to hierarchical structure or structure dependence~\cite{chomsky1957,EVERAERT2015729}, in order to better understand the full range of factors influencing language learnability in both humans and machines. Additionally, incorporating local entropy into model evaluation or as a regularization signal during training could lead to more robust and cognitively plausible language models \cite{timkey-linzen-2023-language,de-varda-marelli-2024-locally}.\looseness-1

In summary, our study presents new evidence of the strong sensitivity of neural LMs to a language's local statistical structure, advancing our understanding of their inductive biases and establishing a foundation for future research on assessing and improving the alignment between artificial and human language processors.

\section*{Limitations}
While our study reveals a strong correlation between \mlocal entropy and LM performance, it is important to note that our analysis remains correlational.
We have not yet pinpointed the precise mechanisms by which variations in local uncertainty impact the learning dynamics of neural language models.
Additionally, our controlled experiments relied on \pfsaAcr{}s to generate languages with varied \mlocal entropy. Although PFSAs provide a tractable framework for such investigations, they capture only a limited set of possible language models.
It is plausible that employing more expressive formalisms, such as pushdown automata or even more powerful models, might reveal different relationships between local entropy and model performance.
In fact, empirical work has repeatedly suggested that some neural LM architectures do not have human-like inductive biases~\citep[][\textit{inter alia}]{mccoy-etal-2020-syntax,yedetore-etal-2023-poor}.\looseness=-1

Furthermore, our focus on information locality, as measured by \mlocal entropy, does not preclude the influence of other inductive biases that may also play significant roles in learning.
Future work will need to disentangle these factors to fully understand their individual and combined effects on neural LMs.\looseness-1

\section*{Ethical considerations}
We employed AI-based tools (ChatGPT and GitHub Copilot) for writing and coding assistance.
These tools were used in compliance with the ACL Policy on the Use of AI Writing Assistance.

\section*{Acknowledgments}
AS was supported by an ETH AI Center Doctoral Fellowship. MG was supported by an ETH Zürich Postdoctoral Fellowship.

\bibliography{custom}

\begin{thebibliography}{36}
\providecommand{\natexlab}[1]{#1}

\bibitem[{Ahuja et~al.(2024)Ahuja, Balachandran, Panwar, He, Smith, Goyal, and
  Tsvetkov}]{Ahuja2024Learning}
Kabir Ahuja, Vidhisha Balachandran, Madhur Panwar, Tianxing He, Noah~A. Smith,
  Navin Goyal, and Yulia Tsvetkov. 2024.
\newblock \href {https://api.semanticscholar.org/CorpusID:269362467} {Learning
  syntax without planting trees: Understanding when and why transformers
  generalize hierarchically}.
\newblock \emph{ArXiv}, abs/2404.16367.

\bibitem[{Borenstein et~al.(2024)Borenstein, Svete, Chan, Valvoda, Nowak,
  Augenstein, Chodroff, and Cotterell}]{borenstein-etal-2024-easy}
Nadav Borenstein, Anej Svete, Robin Chan, Josef Valvoda, Franz Nowak, Isabelle
  Augenstein, Eleanor Chodroff, and Ryan Cotterell. 2024.
\newblock \href {https://arxiv.org/abs/2406.04289} {What languages are easy to
  language-model? {A} perspective from learning probabilistic regular
  languages}.
\newblock \emph{Preprint}, arXiv:2406.04289.

\bibitem[{Butoi et~al.(2025)Butoi, Khalighinejad, Svete, Valvoda, Cotterell,
  and DuSell}]{butoi2025training}
Alexandra Butoi, Ghazal Khalighinejad, Anej Svete, Josef Valvoda, Ryan
  Cotterell, and Brian DuSell. 2025.
\newblock \href {https://openreview.net/forum?id=aWLQTbfFgV} {Training neural
  networks as recognizers of formal languages}.
\newblock In \emph{The Thirteenth International Conference on Learning
  Representations}.

\bibitem[{Charniak et~al.(2000)Charniak, Blaheta, Ge, Hall, Hale, and
  Johnson}]{charniakeugene2000BLLIP}
Eugene Charniak, Don Blaheta, Niyu Ge, Keith Hall, John Hale, and Mark Johnson.
  2000.
\newblock \href {https://doi.org/10.35111/FWEW-DA58} {{{BLLIP}} 1987-89 {{WSJ
  Corpus Release}} 1}.

\bibitem[{Chomsky(1957)}]{chomsky1957}
Noam Chomsky. 1957.
\newblock \href {https://doi.org/10.1515/9783112316009} {\emph{Syntactic
  Structures}}.

\bibitem[{Chomsky et~al.(2023)Chomsky, Roberts, and
  Watumull}]{Chomsky_etal_2023}
Noam Chomsky, Ian Roberts, and Jeffrey Watumull. 2023.
\newblock \href
  {https://www.nytimes.com/2023/03/08/opinion/noam-chomsky-chatgpt-ai.html}
  {Noam chomsky: {T}he false promise of {ChatGPT}}.
\newblock \emph{The New York Times}.

\bibitem[{Choshen et~al.(2024)Choshen, Cotterell, Hu, Linzen, Mueller, Ross,
  Warstadt, Wilcox, Williams, and Zhuang}]{choshen2024callpapers2ndbabylm}
Leshem Choshen, Ryan Cotterell, Michael~Y. Hu, Tal Linzen, Aaron Mueller,
  Candace Ross, Alex Warstadt, Ethan Wilcox, Adina Williams, and Chengxu
  Zhuang. 2024.
\newblock \href {https://arxiv.org/abs/2404.06214} {{T}he 2nd {BabyLM}
  challenge: {S}ample-efficient pretraining on a developmentally plausible
  corpus}.
\newblock \emph{Preprint}, arXiv:2404.06214.

\bibitem[{De~Varda and Marelli(2024)}]{de-varda-marelli-2024-locally}
Andrea De~Varda and Marco Marelli. 2024.
\newblock \href {https://doi.org/10.18653/v1/2024.cmcl-1.3} {Locally biased
  transformers better align with human reading times}.
\newblock In \emph{Proceedings of the Workshop on Cognitive Modeling and
  Computational Linguistics}, pages 30--36, Bangkok, Thailand. Association for
  Computational Linguistics.

\bibitem[{Everaert et~al.(2015)Everaert, Huybregts, Chomsky, Berwick, and
  Bolhuis}]{EVERAERT2015729}
Martin~B.H. Everaert, Marinus~A.C. Huybregts, Noam Chomsky, Robert~C. Berwick,
  and Johan~J. Bolhuis. 2015.
\newblock \href {https://doi.org/10.1016/j.tics.2015.09.008} {Structures, not
  strings: Linguistics as part of the cognitive sciences}.
\newblock \emph{Trends in Cognitive Sciences}, 19(12):729--743.

\bibitem[{Futrell(2019)}]{futrell-2019-information}
Richard Futrell. 2019.
\newblock \href {https://doi.org/10.18653/v1/W19-7902} {Information-theoretic
  locality properties of natural language}.
\newblock In \emph{Proceedings of the First Workshop on Quantitative Syntax
  (Quasy, SyntaxFest 2019)}, pages 2--15, Paris, France. Association for
  Computational Linguistics.

\bibitem[{Futrell(2023)}]{Futrell2023InformationtheoreticPI}
Richard Futrell. 2023.
\newblock \href {https://api.semanticscholar.org/CorpusID:262070319}
  {Information-theoretic principles in incremental language production}.
\newblock \emph{Proceedings of the National Academy of Sciences of the United
  States of America}, 120.

\bibitem[{Futrell et~al.(2020)Futrell, Gibson, and
  Levy}]{futrell-etal-2020-lossy}
Richard Futrell, Edward Gibson, and Roger~P. Levy. 2020.
\newblock \href {https://doi.org/10.1111/cogs.12814} {Lossy-context surprisal:
  An information-theoretic model of memory effects in sentence processing}.
\newblock \emph{Cognitive Science}, 44(3):e12814.

\bibitem[{Futrell and Hahn(2024)}]{Futrell2024LinguisticSF}
Richard Futrell and Michael Hahn. 2024.
\newblock \href {https://api.semanticscholar.org/CorpusID:269921516}
  {Linguistic structure from a bottleneck on sequential information
  processing}.
\newblock \emph{ArXiv}, abs/2405.12109.

\bibitem[{Gibson(1998)}]{gibson1998locality}
Edward Gibson. 1998.
\newblock \href {https://doi.org/10.1016/S0010-0277(98)00034-1} {Linguistic
  complexity: {L}ocality of syntactic dependencies}.
\newblock \emph{Cognition}, 68(1):1--76.

\bibitem[{Gibson(2001)}]{gibson2001dependency}
Edward Gibson. 2001.
\newblock \href {https://doi.org/10.7551/mitpress/3654.003.0008} {The
  dependency locality theory: A distance-based theory of linguistic
  complexity}.
\newblock In \emph{Image, Language, Brain: Papers from the First Mind
  Articulation Project Symposium}. The MIT Press.

\bibitem[{Hahn et~al.(2021)Hahn, Degen, and Futrell}]{hahn_modeling_2021}
Michael Hahn, Judith Degen, and Richard Futrell. 2021.
\newblock \href {https://doi.org/10.1037/rev0000269} {Modeling word and
  morpheme order in natural language as an efficient trade-off of memory and
  surprisal}.
\newblock \emph{Psychological Review}, 128(4):726--756.

\bibitem[{Hahn et~al.(2022)Hahn, Futrell, Levy, and
  Gibson}]{hahn-etal-2022-resource}
Michael Hahn, Richard Futrell, Roger Levy, and Edward Gibson. 2022.
\newblock \href {https://doi.org/10.1073/pnas.2122602119} {A resource-rational
  model of human processing of recursive linguistic structure}.
\newblock \emph{Proceedings of the National Academy of Sciences},
  119(43):e2122602119.

\bibitem[{Hale(2001)}]{hale-2001-probabilistic}
John Hale. 2001.
\newblock \href {https://aclanthology.org/N01-1021/} {A probabilistic {E}arley
  parser as a psycholinguistic model}.
\newblock In \emph{Second Meeting of the North {A}merican Chapter of the
  Association for Computational Linguistics}.

\bibitem[{Heafield(2011)}]{heafield-2011-kenlm}
Kenneth Heafield. 2011.
\newblock \href {https://www.aclweb.org/anthology/W11-2123} {{K}en{LM}: Faster
  and smaller language model queries}.
\newblock In \emph{Proceedings of the Sixth Workshop on Statistical Machine
  Translation}, pages 187--197, Edinburgh, Scotland. Association for
  Computational Linguistics.

\bibitem[{Hochreiter and Schmidhuber(1997)}]{hochreiter-schmidhuber-1997-lstm}
Sepp Hochreiter and J\"{u}rgen Schmidhuber. 1997.
\newblock \href {https://doi.org/10.1162/neco.1997.9.8.1735} {Long short-term
  memory}.
\newblock \emph{Neural Comput.}, 9(8):1735–1780.

\bibitem[{Hu et~al.(2020)Hu, Gauthier, Qian, Wilcox, and Levy}]{Hu_et_al_2020}
Jennifer Hu, Jon Gauthier, Peng Qian, Ethan Wilcox, and Roger Levy. 2020.
\newblock \href {https://doi.org/10.18653/v1/2020.acl-main.158} {A systematic
  assessment of syntactic generalization in neural language models}.
\newblock In \emph{Proceedings of the 58th Annual Meeting of the Association
  for Computational Linguistics}, pages 1725--1744, Online. Association for
  Computational Linguistics.

\bibitem[{Kallini et~al.(2024)Kallini, Papadimitriou, Futrell, Mahowald, and
  Potts}]{kallini-etal-2024-mission}
Julie Kallini, Isabel Papadimitriou, Richard Futrell, Kyle Mahowald, and
  Christopher Potts. 2024.
\newblock \href {https://doi.org/10.18653/v1/2024.acl-long.787} {Mission:
  Impossible language models}.
\newblock In \emph{Proceedings of the 62nd Annual Meeting of the Association
  for Computational Linguistics (Volume 1: Long Papers)}, pages 14691--14714,
  Bangkok, Thailand. Association for Computational Linguistics.

\bibitem[{Kingma and Ba(2015)}]{kingma-ba-2015-adam}
Diederik~P. Kingma and Jimmy~Lei Ba. 2015.
\newblock \href {https://arxiv.org/abs/1412.6980} {{A}dam: A method for
  stochastic optimization}.
\newblock In \emph{The Third International Conference for Learning
  Representations}, San Diego, California, USA.

\bibitem[{Levy(2008)}]{levy-2008}
Roger Levy. 2008.
\newblock \href {https://doi.org/10.1016/j.cognition.2007.05.006}
  {Expectation-based syntactic comprehension}.
\newblock \emph{Cognition}, 106(3).

\bibitem[{Malagutti et~al.(2024)Malagutti, Buinovskij, Svete, Meister, Amini,
  and Cotterell}]{malagutti-etal-2024-role}
Luca Malagutti, Andrius Buinovskij, Anej Svete, Clara Meister, Afra Amini, and
  Ryan Cotterell. 2024.
\newblock \href {https://doi.org/10.18653/v1/2024.naacl-long.382} {The role of
  $n$-gram smoothing in the age of neural networks}.
\newblock In \emph{Proceedings of the 2024 Conference of the North American
  Chapter of the Association for Computational Linguistics: Human Language
  Technologies (Volume 1: Long Papers)}, pages 6882--6899, Mexico City, Mexico.
  Association for Computational Linguistics.

\bibitem[{McCoy et~al.(2020)McCoy, Frank, and Linzen}]{mccoy-etal-2020-syntax}
R.~Thomas McCoy, Robert Frank, and Tal Linzen. 2020.
\newblock \href {https://doi.org/10.1162/tacl_a_00304} {Does syntax need to
  grow on trees? sources of hierarchical inductive bias in sequence-to-sequence
  networks}.
\newblock \emph{Transactions of the Association for Computational Linguistics},
  8:125--140.

\bibitem[{Mitchell(1980)}]{mitchell1980need}
Tom~M. Mitchell. 1980.
\newblock \href
  {https://www.cs.cmu.edu/afs/cs.cmu.edu/user/mitchell/ftp/pubs/NeedForBias_1980.pdf}
  {The need for biases in learning generalizations}.
\newblock Technical Report CBM-TR 5-110, Rutgers University, New Brunswick, New
  Jersey, USA.

\bibitem[{Opedal et~al.(2024)Opedal, Chodroff, Cotterell, and
  Wilcox}]{opedal-etal-2024-role}
Andreas Opedal, Eleanor Chodroff, Ryan Cotterell, and Ethan Wilcox. 2024.
\newblock \href {https://doi.org/10.18653/v1/2024.emnlp-main.179} {On the role
  of context in reading time prediction}.
\newblock In \emph{Proceedings of the 2024 Conference on Empirical Methods in
  Natural Language Processing}, pages 3042--3058, Miami, Florida, USA.
  Association for Computational Linguistics.

\bibitem[{Paszke et~al.(2019)Paszke, Gross, Massa, Lerer, Bradbury, Chanan,
  Killeen, Lin, Gimelshein, Antiga, Desmaison, Kopf, Yang, DeVito, Raison,
  Tejani, Chilamkurthy, Steiner, Fang, Bai, and
  Chintala}]{paszke-etal-2019-pytorch}
Adam Paszke, Sam Gross, Francisco Massa, Adam Lerer, James Bradbury, Gregory
  Chanan, Trevor Killeen, Zeming Lin, Natalia Gimelshein, Luca Antiga, Alban
  Desmaison, Andreas Kopf, Edward Yang, Zachary DeVito, Martin Raison, Alykhan
  Tejani, Sasank Chilamkurthy, Benoit Steiner, Lu~Fang, Junjie Bai, and Soumith
  Chintala. 2019.
\newblock \href
  {https://papers.nips.cc/paper_files/paper/2019/hash/bdbca288fee7f92f2bfa9f7012727740-Abstract.html}
  {{P}y{T}orch: An imperative style, high-performance deep learning library}.
\newblock In \emph{Advances in Neural Information Processing Systems},
  volume~32. Curran Associates, Inc.

\bibitem[{Piantadosi(2024)}]{piantadosi_modern_2024}
Steven Piantadosi. 2024.
\newblock \href {https://lingbuzz.net/lingbuzz/007180} {Modern language models
  refute {Chomsky}’s approach to language}.
\newblock LingBuzz Published In: Edward Gibson \& Moshe Poliak (eds.), From
  fieldwork to linguistic theory: A tribute to Dan Everett (Empirically
  Oriented Theoretical Morphology and Syntax 15), 353–414. Berlin: Language
  Science Press. https : //doi.org/10.5281/zenodo.12665933.

\bibitem[{Rawski and Heinz(2019)}]{Rawski2019NoFL}
Jonathan Rawski and Jeffrey Heinz. 2019.
\newblock \href {https://api.semanticscholar.org/CorpusID:267932762} {No free
  lunch in linguistics or machine learning: Response to pater}.
\newblock \emph{Language}, 95:e125 -- e135.

\bibitem[{Shannon(1948)}]{shannon1948}
C.~E. Shannon. 1948.
\newblock \href {https://doi.org/10.1002/j.1538-7305.1948.tb01338.x} {A
  mathematical theory of communication}.
\newblock \emph{The Bell System Technical Journal}, 27(3):379--423.

\bibitem[{Timkey and Linzen(2023)}]{timkey-linzen-2023-language}
William Timkey and Tal Linzen. 2023.
\newblock \href {https://doi.org/10.18653/v1/2023.findings-emnlp.582} {A
  language model with limited memory capacity captures interference in human
  sentence processing}.
\newblock In \emph{Findings of the Association for Computational Linguistics:
  EMNLP 2023}, pages 8705--8720, Singapore. Association for Computational
  Linguistics.

\bibitem[{Vaswani et~al.(2017)Vaswani, Shazeer, Parmar, Uszkoreit, Jones,
  Gomez, Kaiser, and Polosukhin}]{vaswani-etal-2017-attention}
Ashish Vaswani, Noam Shazeer, Niki Parmar, Jakob Uszkoreit, Llion Jones,
  Aidan~N. Gomez, {\L}ukasz Kaiser, and Illia Polosukhin. 2017.
\newblock \href
  {https://papers.nips.cc/paper_files/paper/2017/hash/3f5ee243547dee91fbd053c1c4a845aa-Abstract.html}
  {Attention is all you need}.
\newblock In \emph{Advances in Neural Information Processing Systems},
  volume~30. Curran Associates, Inc.

\bibitem[{Weiss et~al.(2018)Weiss, Goldberg, and
  Yahav}]{weiss-etal-2018-practical}
Gail Weiss, Yoav Goldberg, and Eran Yahav. 2018.
\newblock \href {https://doi.org/10.18653/v1/P18-2117} {On the practical
  computational power of finite precision {RNN}s for language recognition}.
\newblock In \emph{Proceedings of the 56th Annual Meeting of the Association
  for Computational Linguistics (Volume 2: Short Papers)}, pages 740--745,
  Melbourne, Australia. Association for Computational Linguistics.

\bibitem[{Yedetore et~al.(2023)Yedetore, Linzen, Frank, and
  McCoy}]{yedetore-etal-2023-poor}
Aditya Yedetore, Tal Linzen, Robert Frank, and R.~Thomas McCoy. 2023.
\newblock \href {https://doi.org/10.18653/v1/2023.acl-long.521} {How poor is
  the stimulus? {E}valuating hierarchical generalization in neural networks
  trained on child-directed speech}.
\newblock In \emph{Proceedings of the 61st Annual Meeting of the Association
  for Computational Linguistics (Volume 1: Long Papers)}, pages 9370--9393,
  Toronto, Canada. Association for Computational Linguistics.

\end{thebibliography}
\bibliographystyle{acl_natbib}

\appendix
\onecolumn

\section{Probabilistic Finite-State Automata} \label{sec:pfsa_detail}
Before listing a number of useful results for computing quantities of interest in \pfsaAcr{}s, we list a few relevant definitions.
\begin{definition}
Let $\wfsa = \wfsatuple$ be a \pfsaAcr{}.
We define the \defn{transition matrix} $\tranMtx \in \R^{\nstates \times \nstates}$ of $\wfsa$ as
the matrix containing the probabilities of transitioning from state $\stateq_i \in \states$ to state $\stateq_j \in \states$ in $\wfsa$ with any $\sym \in \alphabet$:
\begin{equation}\label{eq:M-mtx}
    \etranMtx_{i, j} \defeq \sum_{\sym \in \alphabet} \sum_{\edge{\stateq_i}{\sym}{w}{\stateq_j}\,\in\,\trans} w,
\end{equation}
where we fix some arbitrary enumeration of states $(\stateq_1, \ldots, \stateq_{\nstates})$.
We also define the \defn{symbol-specific transition matrix} $\tranMtx^{\left(\sym\right)}$ where $\etranMtx^{(\sym)}_{i,j}$ is the probability of transitioning from state $\stateq_i \in \states$ to state $\stateq_j \in \states$ in $\wfsa$ with a \sym-labeled transition:
\begin{equation}
    \esTranMtx{\sym}_{i,j} \defeq \sum_{\edge{\stateq_i}{\sym}{w}{\stateq_j}\,\in\,\trans} w.
\end{equation}
We naturally extend this definition to strings and define for $\str = \sym_1 \cdots \sym_\strlen$:
\begin{equation}
    \sTranMtx{\str} \defeq \sTranMtx{\sym_1} \cdots \sTranMtx{\sym_\strlen}.
\end{equation}
Furthermore, we define \defn{Kleene closure} of $\tranMtx$ as
\begin{align}
    \kleene{\tranMtx} \defeq \sum_{n=0}^{\infty} \tranMtx^{n},
\end{align}
where the series above converges when the spectral radius satisfies $\rho(\tranMtx) < 1$.
Additionally, when the $(\mI - \tranMtx)$ is invertible, this is simply calculated as
\begin{align}
    \kleene{\tranMtx} = (\mI - \tranMtx)^{-1}.
\end{align}

\end{definition}
\begin{remark} \label{rem:stringsum}
    It is a standard exercise to show that $\esTranMtx{\str}_{i, j}$ equals the sum of the weights of $\str$-scanning strings from $\stateq_i$ to $\stateq_j$.
\end{remark}
\begin{definition} \label{def:emission-mtx}
    Let $\wfsa = \wfsatuple$ be a \pfsaAcr{}.
    The \defn{emission matrix} $\mE \in \R^{\nstates \times \nsymbols}$ is defined by
    \begin{equation} \label{eq:emission-mtx}
        \emE_{i, k} \defeq \sum_{\edge{\stateq_i}{\sym_k}{w}{\stateq'} \in \trans} w.
    \end{equation}
\end{definition}
For a \pfsaAcr{} $\wfsa$ and a path $\apath = \edge{\stateq_0}{\sym_1}{w_1}{} \cdots \edge{}{\sym_N}{w_N}{\stateq_N} \in \paths\left(\wfsa\right)$, we write $\prevqFun{\apath} \defeq \stateq_0$ for the initial state of the path and $\nextqFun{\apath} \defeq \stateq_N$ for its final state.
We define the path prefix random variable $\rvpathpfx$ distributed as
\begin{equation}
    \prob\left(\rvpathpfx = \apath \right) \propto \initfFun{\prevqFun{\apath}} \innerpathweightFun{\apath}.
\end{equation}
This is analogous to prefix string probabilities, and the distribution is normalizable exactly when prefix probabilities are.
Similarly, we define $\rvpathinfx$, which is distributed as
\begin{equation}
    \prob\left(\rvpathinfx = \apath \right) \propto \sum_{\substack{\apath' \in \paths\left(\wfsa\right) \\ \nextqFun{\apath'} = \prevqFun{\apath}}} \initfFun{\prevqFun{\apath'}} \innerpathweightFun{\apath'} \innerpathweightFun{\apath},
\end{equation}
which is analogous to string infix probabilities.

\pfsaAcr{}s are particularly attractive to study since they allow us to exactly compute many interesting quantities efficiently.
In the following section, we describe how one can compute the \mlocal entropy of the language model defined by a \pfsaAcr{}.

\subsection{Useful Properties of Probabilistic Finite-State Automata}
The following lemmata hold for a general \pfsaAcr{} $\wfsa = \wfsatuple$ and the language model $\pLMwfsa$ induced by it.
None of the results are novel, but we include full proofs for completeness.
\begin{lemma}[Computing the probability of a string with a \pfsaAcr] \label{lem:pfsa-str-prob}
    The probability of $\str \in \klAlphabet$ is:
    \begin{equation}
        \pLMwfsaFun{\str} = \vlambda^\top \sTranMtx{\str} \vrho.
    \end{equation}
\end{lemma}
\begin{proof}
    We know from \cref{rem:stringsum} that $\sTranMtx{\str}_{i, j}$ corresponds to the sum of the path weights from $\stateq_i$ to $\stateq_j$.
    Multiplying each entry with the source state's initial weight and the target state's final weight, we arrive at the result.
\end{proof}

\begin{lemma}[Computing the prefix probability of a string] \label{lem:pfsa-prefix-prob}
    The prefix probability of $\str \in \klAlphabet$ is:
    \begin{equation}
        \pLMwfsaprefixFun{\str} = \vlambda^\top \sTranMtx{\str} \kleene{\tranMtx} \vrho. %
    \end{equation}
\end{lemma}
\begin{proof}
    \begin{subequations}
    \begin{align}
        \pLMwfsaprefixFun{\str} &=
        \sum_{\str' \in \klAlphabet} \pLMwfsaFun{\str \str'} \\
        &= \sum_{\str' \in \klAlphabet} \vlambda^\top \sTranMtx{\str \str'} \vrho \justification{\cref{lem:pfsa-str-prob}} \\
        &= \sum_{\str' \in \klAlphabet} \vlambda^\top \sTranMtx{\str} \sTranMtx{\str'} \vrho \\
        &= \vlambda^\top \sTranMtx{\str} \left(\sum_{\str' \in \klAlphabet} \sTranMtx{\str'}\right) \vrho \\
        &= \vlambda^\top \sTranMtx{\str} \kleene{\tranMtx} \vrho
    \end{align}
\end{subequations}
\end{proof}

\begin{lemma}[Computing the next-symbol distribution] \label{lem:next-symbol-dist}
    Let $\str \in \klAlphabet$.
    The distribution over the next symbols after observing $\str$ is
    \begin{equation}
        \probFun{\sym_k \mid \yieldFun{\rvpathpfx} = \str} = \frac{\left(\vlambda^\top \sTranMtx{\str} \mE\right)_k}{\pLMwfsaprefixFun{\str}}.
    \end{equation}
\end{lemma}
\begin{proof}
    \begin{subequations}
        \begin{align}
            \probFun{\sym_k \mid \yieldFun{\rvpathpfx} = \str}             &= \frac{\prob\left(\sym_k, \yieldFun{\rvpathpfx} = \str\right)}{\prob\left(\yieldFun{\rvpathpfx} = \str\right)} \\
            &= \frac{1}{\pLMwfsaprefixFun{\str}} \sum_{\apath \in \paths(\wfsa, \str)} \initfFun{\prevqFun{\apath}} \; \innerpathweightFun{\apath} \; \prob\left(\sym_k \mid \nextqFun{\apath} \right) \justification{Summing over all $\str$-yielding paths} \\
            &= \frac{1}{\pLMwfsaprefixFun{\str}} \sum_{j = 1}^{\nstates} \prob\left(\sym_k \mid \stateq_j \right) \sum_{\substack{\apath \in \paths(\wfsa, \str) \\ \nextqFun{\apath} = \stateq_j}} \initfFun{\prevqFun{\apath}} \; \innerpathweightFun{\apath} \\
            &= \frac{1}{\pLMwfsaprefixFun{\str}} \sum_{j = 1}^{\nstates} \prob\left(\sym_k \mid \stateq_j \right) \sum_{i = 1}^{\nstates}  \initfFun{\stateq_i} \sum_{\substack{\apath \in \paths(\wfsa, \str) \\ \prevqFun{\apath} = \stateq_i, \nextqFun{\apath} = \stateq_j}} \; \innerpathweightFun{\apath} \\
            &= \frac{1}{\pLMwfsaprefixFun{\str}} \sum_{j = 1}^{\nstates} \prob\left(\sym_k \mid \stateq_j \right) \sum_{i = 1}^{\nstates}  \initfFun{\stateq_i} \sTranMtx{\str}_{i,j} \\
            &= \frac{1}{\pLMwfsaprefixFun{\str}} \sum_{j = 1}^{\nstates} \prob\left(\sym_k \mid \stateq_j \right) \left(\vlambda^\top \sTranMtx{\str}\right)_{j} \\
            &= \frac{1}{\pLMwfsaprefixFun{\str}} \sum_{j = 1}^{\nstates} \left(\vlambda^\top \sTranMtx{\str}\right)_{j} w \justification{$\edge{\stateq_i}{\sym_k}{w}{\stateq'} \in \trans$} \\
            &= \frac{1}{\pLMwfsaprefixFun{\str}} \left(\vlambda^\top \sTranMtx{\str} \mE \right)_k \justification{\cref{eq:emission-mtx}}
        \end{align}
    \end{subequations}
\end{proof}

\begin{lemma}[Computing the infix probability of a string] \label{lem:pfsa-infix}
    The infix probability of $\str \in \klAlphabet$ is:
    \begin{equation}
        \pLMwfsainfixFun{\str} = \vlambda^\top \kleene{\tranMtx} \sTranMtx{\str} \kleene{\tranMtx} \vrho. %
    \end{equation}
\end{lemma}
\begin{proof}
    \begin{subequations}
    \begin{align}
        \pLMwfsainfixFun{\str} &=
        \sum_{\str' \in \klAlphabet} \pLMwfsaprefixFun{\str' \str} \\
        &= \sum_{\str' \in \klAlphabet} \vlambda^\top \sTranMtx{\str' \str} \kleene{\tranMtx} \vrho \justification{\cref{lem:pfsa-prefix-prob}} \\
        &= \vlambda^\top \left(\sum_{\str' \in \klAlphabet} \sTranMtx{\str'} \right) \sTranMtx{\str} \kleene{\tranMtx} \vrho \\
        &= \vlambda^\top \kleene{\tranMtx} \sTranMtx{\str} \kleene{\tranMtx} \vrho
    \end{align}
\end{subequations}
\end{proof}

\begin{lemma}[Computing the infix next-symbol distribution] \label{lem:pfsa-infix-next-symbol}
    Let $\ctx \in \mAlphabet$.
    The distribution over the next symbols after observing $\ctx$ as the last $m - 1$ symbols is
    \begin{equation}
        \probFun{\sym_k \mid \yieldFun{\rvpathinfx} = \ctx} = \frac{\left(\vlambda^\top \kleene{\tranMtx} \sTranMtx{\ctx} \mE\right)_k}{\pLMwfsainfixFun{\ctx}}.
    \end{equation}
\end{lemma}
\begin{proof}
    \begin{subequations}
        \begin{align}
            \probFun{\sym_k \mid \yieldFun{\rvpathinfx} = \ctx}
            &= \frac{\prob\left(\sym_k, \yieldFun{\rvpathinfx} = \ctx\right)}{\prob\left(\yieldFun{\rvpathinfx} = \ctx\right)} \\
            &= \frac{1}{\pLMwfsainfixFun{\ctx}} \sum_{\str' \in \klAlphabet} \sum_{\apath \in \paths(\wfsa, \str' \ctx)} \initfFun{\prevqFun{\apath}} \; \innerpathweightFun{\apath} \; \prob\left(\sym_k \mid \nextqFun{\apath} \right) \justification{Summing over all $\ctx$-ending paths} \\
            &= \frac{1}{\pLMwfsainfixFun{\ctx}} \sum_{j = 1}^{\nstates} \prob\left(\sym_k \mid \stateq_j \right) \sum_{\str' \in \klAlphabet} \sum_{\substack{\apath \in \paths(\wfsa, \str' \ctx) \\ \nextqFun{\apath} = \stateq_j}} \initfFun{\prevqFun{\apath}} \; \innerpathweightFun{\apath} \\
            &= \frac{1}{\pLMwfsainfixFun{\ctx}} \sum_{j = 1}^{\nstates} \prob\left(\sym_k \mid \stateq_j \right) \sum_{i = 1}^{\nstates} \sum_{\str' \in \klAlphabet} \initfFun{\stateq_i} \sum_{\substack{\apath \in \paths(\wfsa, \str' \ctx) \\ \prevqFun{\apath} = \stateq_i, \nextqFun{\apath} = \stateq_j}} \; \innerpathweightFun{\apath} \\
            &= \frac{1}{\pLMwfsainfixFun{\ctx}} \sum_{j = 1}^{\nstates}  \prob\left(\sym_k \mid \stateq_j \right) \sum_{i = 1}^{\nstates} \sum_{\str' \in \klAlphabet} \initfFun{\stateq_i} \left(\sTranMtx{\str'} \sTranMtx{\ctx}\right)_{i,j} \\
            &= \frac{1}{\pLMwfsainfixFun{\ctx}} \sum_{j = 1}^{\nstates}  \prob\left(\sym_k \mid \stateq_j \right) \sum_{i = 1}^{\nstates} \initfFun{\stateq_i} \sum_{\str' \in \klAlphabet}  \left(\sTranMtx{\str'} \sTranMtx{\ctx}\right)_{i,j} \\
            &= \frac{1}{\pLMwfsainfixFun{\ctx}} \sum_{j = 1}^{\nstates}  \prob\left(\sym_k \mid \stateq_j \right) \sum_{i = 1}^{\nstates} \initfFun{\stateq_i} \left(\kleene{\tranMtx} \sTranMtx{\ctx}\right)_{i,j} \\
            &= \frac{1}{\pLMwfsainfixFun{\ctx}} \sum_{j = 1}^{\nstates} \prob\left(\sym_k \mid \stateq_j \right) \left(\vlambda^\top \kleene{\tranMtx} \sTranMtx{\str}\right)_{j} \\
            &= \frac{1}{\pLMwfsainfixFun{\ctx}} \sum_{j = 1}^{\nstates} \left(\vlambda^\top \kleene{\tranMtx} \sTranMtx{\str}\right)_{j} w \justification{$\edge{\stateq_i}{\sym_k}{w}{\stateq'} \in \trans$} \\
            &= \frac{1}{\pLMwfsainfixFun{\ctx}} \left(\vlambda^\top \kleene{\tranMtx} \sTranMtx{\str} \mE \right)_k \justification{\cref{eq:emission-mtx}}
        \end{align}
    \end{subequations}
\end{proof}

\begin{lemma}[Computing the \mlocal entropy of a \dpfsaAcr{}]
    The \mlocal entropy of $\pLMwfsa$ can be computed in time $\bigOFun{m\nstates^3\nsymbols^{m - 1}}$.
\end{lemma}
\begin{proof}
    The \mlocal entropy of $\pLMwfsa$ can be computed as
    \begin{subequations}
        \begin{align} \label{eq:local-ent-defn}
            \mlocalentropy\left(\pLMwfsa\right)
            &= \underset{\ctx \sim \prob(\rvctx = \ctx)}{\E}\bigl[\entropyFun{{\rvContextNextSymbol \mid \rvctx = \ctx}}\bigr] \\
            &= \frac{1}{\infixnorm}\sum_{\ctx \in \mAlphabet} \pLMwfsainfixFun{\ctx} \entropyFun{{\rvContextNextSymbol \mid \rvctx = \ctx}} \label{eq:infix-derivation}
        \end{align}
    \end{subequations}
    The terms $\pLMwfsainfixFun{\ctx}$ and $\entropyFun{\rvContextNextSymbol \mid \rvctx = \ctx}$ can be computed in time $\bigOFun{m \nstates^3}$ as per \cref{lem:pfsa-infix,lem:pfsa-infix-next-symbol} for each $\ctx \in \mAlphabet$.
    Computing this for each $\ctx$ individually, we arrive at the claimed complexity.
\end{proof}

\clearpage
\section{Generating Random \pfsaAcr{}s} \label{sec:pfsa-generation}
\cref{alg:random_dpfsa} and the subprocedure in \cref{alg:generate_outgoing_labels} describe our \pfsaAcr{} generation procedure.

\begin{algorithm*}[h]
\caption{Generate a Random DPFSA.\\[1mm]
\textbf{Input:}
$\nstates$ (number of states), $\nsymbols$ (number of symbols), $\meanStrLen$ (target mean string length), $R_T$ (topology random generator), and $R_W$ (weight random generator).\\[1mm]
\textbf{Output:} A PFSA $A$ with randomly assigned transitions and normalized weights with $\nstates$ states and $\nsymbols$ symbols.\\[1mm]
\textbf{Note:} $\textsc{Choice}(R, S, 1)$ denotes selecting one element uniformly at random from the set $S$ using the random generator $R$.\\
$\texttt{unused}$ is initialized as $\states$, and $\texttt{out-arcs}$ is a mapping that assigns to each state a subset of $\alphabet$ (the allowed outgoing symbols).\\[1mm]
For exponential sampling, we write $w \sim \operatorname{Exp}(0.1)$ to denote that $w$ is drawn from an exponential distribution with rate $0.1$, i.e., with density $f(w)=0.1\,e^{-0.1w}$ for $w\ge 0$.
}\label{alg:random_dpfsa}
\begin{algorithmic}[1]
\Function{RandomDPFSA}{$\nstates,\ \nsymbols,\ \meanStrLen,\ R_T,\ R_W$}
    \State $\qinit \gets \Call{Choice}{R_T, \states, 1}$
    \State \textbf{Initialize} $\wfsa \gets \wfsatuple$
    \State \textbf{Initialize} $\vlambda \gets \mathbf{0}_{\nstates}$ and \textbf{set} $\initfFun{\qinit} \gets 1$
    \State \textbf{Initialize} $\sTranMtx{\sym}$ to a $\nstates \times \nstates$ matrix of zeros for $\sym \in \alphabet$
    \State $\texttt{unused} \gets \states$
    \State $\texttt{state-outgoing-symbols} \gets \Call{GetOutgoingSymbols}{\states,\ \alphabet,\ R_T}$

    \For{$\stateq \in \states$}
        \For{$\sym \in \alphabet$}
            \If{$\texttt{unused} \neq \emptyset$}
                \State $\stateq' \gets \Call{Choice}{R_T, \texttt{unused}, 1}$
                \State Remove $\stateq'$ from $\texttt{unused}$
            \Else
                \State $\stateq' \gets \Call{Choice}{R_T, \states, 1}$
            \EndIf
            \State \textbf{Let} $w \sim \operatorname{Exp}(0.1)$
            \State $\esTranMtx{\sym}_{\stateq, \stateq'} \gets w \cdot \ind{\sym \in \texttt{state-outgoing-symbols}[\stateq]} + 0.001$
        \EndFor
    \EndFor

    \For{$\stateq \in \states$} \Comment{Set final weights and normalize outgoing weights for each state.}
        \State $t \gets \sum_{\sym=0}^{\nsymbols-1} \text{sum}(\sTranMtx{\sym}_{\stateq, :})$
        \State $\finalfFun{\stateq} \gets t / \meanStrLen$
        \State $s \gets t + \finalfFun{\stateq}$
        \For{$\sym \in \{0, \ldots, \nsymbols-1\}$}
            \State $\sTranMtx{\sym}_{\stateq, :} \gets \sTranMtx{\sym}_{\stateq, :} / s$
        \EndFor
        \State $\finalfFun{\stateq} \gets \finalfFun{\stateq} / s$
    \EndFor
    \State \Return $\wfsa$
\EndFunction
\end{algorithmic}
\end{algorithm*}

\begin{algorithm*}[h]
\caption{Generate Outgoing Symbols for Each State.\\[1mm]
\textbf{Input:}
$\nstates$ (number of states), $\nsymbols$ (number of symbols), $R$ (random generator), and $s_{min}$ (min. unique symbols per state; default: 2).\\[1mm]
\textbf{Output:} A list $S$ of $\nstates$ sets, each containing outgoing symbols.\\[1mm]
\textbf{Note:} $\textsc{Choice}(R, X, k)$ selects $k$ distinct elements uniformly at random from $X$ using $R$, and $\textsc{Integers}(R, a, b)$ returns a random integer in $[a, b)$.
}\label{alg:generate_outgoing_labels}
\begin{algorithmic}[1]
\Function{GetOutgoingSymbols}{$\states, \alphabet, R, s_{min}$}
    \State \textbf{Initialize} $\texttt{state-outgoing-symbols} \gets$ an array of $\nstates$ empty sets
    \For{$q \in \states$}
        \State $\texttt{s} \gets \Call{Choice}{R, \mathcal{N}, \min(s_{min}, \nsymbols)}$ \Comment{Assign each state at least $s_{min}$ symbols.}
        \State $\texttt{state-outgoing-symbols}[q] \gets \texttt{state-outgoing-symbols}[q] \cup \texttt{s}$
    \EndFor
    \For{$\sym \in \alphabet$} \Comment{Ensure each symbol appears in at least one set.}
        \State $\stateq \gets \Call{Choice}{R, \states, 1}$
        \State \textbf{Add} $\sym$ to \texttt{state-outgoing-symbols}[q]
    \EndFor
    \State $\texttt{M} \gets \displaystyle \sum_{q=0}^{\nstates-1} \Call{Integers}{R, 0, \max(1, \lfloor \nsymbols/2 \rfloor - s_{min})}$
    \For{$j \gets 1$ \textbf{to} $\texttt{M}$} \Comment{Add random transitions.}
        \State $\sym \gets \Call{Choice}{R, \alphabet, 1}$
        \State $\stateq \gets \Call{Choice}{R, \states, 1}$
        \State \textbf{Add} $\sym$ to \texttt{state-outgoing-symbols}[q]
    \EndFor
    \State \Return $\texttt{state-outgoing-symbols}$
\EndFunction
\end{algorithmic}
\end{algorithm*}

\section{Details of Neural Language Models} \label{sec:architectures}
\subsection{Transformer}
We use a 4-layer causally-masked Transformer with 768-dimensional embeddings, 3,072-dimensional feedforward layers, and 12 attention heads, implemented in PyTorch.
Following \citet{vaswani-etal-2017-attention}, we map input symbols to vectors of size 768 with a scaled embedding layer and add sinusoidal positional encodings.
We use pre-norm instead of post-norm and apply layer norm to the output of the last layer.
We use the same dropout rate throughout the Transformer. We apply it in the same places as \citet{vaswani-etal-2017-attention}, and, as implemented by PyTorch, we also apply it to the hidden units of feedforward sublayers and to the attention probabilities of scaled dot-product attention operations. We always use $\bos$ as the first input symbol to the Transformer.

\subsection{LSTM}
We use a single-layer LSTM~\cite{hochreiter-schmidhuber-1997-lstm} with 512-dimensional hidden units, implemented in PyTorch with some modifications as in \citet{butoi2025training}.
{\newcommand{\lstmaffine}[1]{\affinel{#1}{\thelayerno}{\begin{bmatrix} \thehiddenstatedropoutli{\thelayerno-1}{\thetimestep} \\ \thehiddenstateli{\thelayerno}{\thetimestep-1} \end{bmatrix}}}
\newcommand{\bounds}{(1 \leq \thelayerno \leq \thenumlayers; 1 \leq \thetimestep \leq \thelength)}
\begin{subequations}
\begin{align}
    \thehiddenstateli{0}{\thetimestep} &\defeq \theinputembeddingi{\thetimestep} = \theembeddingmatrix_{\thesymboli{\thetimestep}} && (1 \leq \thetimestep \leq \thelength) \\
    \thehiddenstateli{\thelayerno}{0} &\defeq \tanh(\initialhiddenstateparaml{\thelayerno}) \quad && (1 \leq \thelayerno \leq \thenumlayers) \\
    \thehiddenstatedropoutli{\thelayerno}{\thetimestep} &\defeq \dropout{\thehiddenstateli{\thelayerno}{\thetimestep}} && (0 \leq \thelayerno \leq \thenumlayers; 0 \leq \thetimestep \leq \thelength) \\
    \theinputgateli{\thelayerno}{\thetimestep} &\defeq \logistic{\lstmaffine{i}} \quad && \bounds \\
    \theforgetgateli{\thelayerno}{\thetimestep} &\defeq \logistic{\lstmaffine{f}} \quad && \bounds \\
    \thecandidateli{\thelayerno}{\thetimestep} &\defeq \tanh(\lstmaffine{g}) \quad && \bounds \\
    \theoutputgateli{\thelayerno}{\thetimestep} &\defeq \logistic{\lstmaffine{o}} \quad && \bounds \\
    \thememorycellli{\thelayerno}{\thetimestep} &\defeq \theforgetgateli{\thelayerno}{\thetimestep} \elementwisemultiply \thememorycellli{\thelayerno}{\thetimestep-1} + \theinputgateli{\thelayerno}{\thetimestep} \elementwisemultiply \thecandidateli{\thelayerno}{\thetimestep} \quad && \bounds \\
    \thehiddenstateli{\thelayerno}{\thetimestep} &\defeq \theoutputgateli{\thelayerno}{\thetimestep} \elementwisemultiply \tanh(\thememorycellli{\thelayerno}{\thetimestep}) \quad && \bounds \\
    \thememorycellli{\thelayerno}{0} &\defeq \vzero \quad && (1 \leq \thelayerno \leq \thenumlayers) \\
    \thehiddeni{\thetimestep} &\defeq \thehiddenstatedropoutli{\thenumlayers}{\thetimestep} && (0 \leq \thetimestep \leq \thelength)
\end{align}
\end{subequations}

Here, $\theembeddingmatrix$ is an embedding matrix to map each symbol $\thesymboli{\thetimestep}$ of the input string to an embedding $\theinputembeddingi{\thetimestep} = \theembeddingmatrix_{\thesymboli{\thetimestep}}$.
The size of the embeddings is always $\thehiddensize$ (the size of the hidden vectors), and we denote the number of layers in the model as $\thenumlayers$.
Also, $\elementwisemultiply$ denotes elementwise multiplication, and $\dropout{\cdot}$ indicates the application of dropout.
Here, $\initialhiddenstateparaml{\thelayerno} \in \realset^{\thehiddensize}$ is a learned parameter, making the initial hidden state $\thehiddenstateli{\thelayerno}{0}$ of each layer learned.
A modification is made from the original PyTorch implementation: each pair of $b_{ii}$ and $b_{hi}$, $b_{if}$ and $b_{hf}$, $b_{ig}$ and $b_{hg}$, and $b_{io}$ and $b_{ho}$ is replaced with a single bias parameter per layer.

\section{Hyperparameters for Neural Language Model Training} \label{sec:experiments-details}

Wherever dropout is applicable, we use a dropout rate of 0.1. For layer norm, we initialize weights to 1 and biases to 0. We initialize all other parameters by sampling uniformly from $[-0.1, 0.1]$.

For each epoch, we randomly shuffle the training set and group strings of similar lengths into the same minibatch, enforcing an upper limit of 2,048 symbols per batch, including padding, $\bos$, and $\eos$ symbols.
We train each model by minimizing cross-entropy on the validation set using Adam \citep{kingma-ba-2015-adam}.
We clip gradients with a threshold of 5 using $\normltwo$ norm rescaling.
We take a checkpoint every 10K examples, at which point we evaluate the model on the validation set and update the learning rate and early stopping schedules.
We multiply the learning rate by 0.5 after 5 checkpoints of no decrease in cross-entropy on the validation set, and we stop early after 10 checkpoints of no decrease.
We select the checkpoint with the lowest cross-entropy on the validation set when reporting results.
We train for a maximum of 1K epochs.

\section{Pearson Correlation Coefficients between \Mlocal Entropy and KL Divergence}
\Cref{tab:mlocal_entropy_pfsa_correlation} reports the Pearson correlation coefficients between the \mlocal entropy of \pfsaAcr{} and the estimated KL divergence ($\KLhat$; \Cref{sec:exp2-evaluating}) in \Cref{sec:exp_pfsa}.

\begin{table*}[h]
    \small
    \centering
    \begin{tabular}{llccc|ccc|ccc}
    \toprule
     & $\nstates$ & \multicolumn{3}{c}{16} & \multicolumn{3}{c}{24} & \multicolumn{3}{c}{32} \\
     & $\nsymbols$ & 32 & 48 & 64 & 32 & 48 & 64 & 32 & 48 & 64 \\
    \textsc{Architecture} & \textsc{m} &  &  &  &  &  &  &  &  &  \\
    \midrule
    \multirow[c]{4}{*}{\textsc{LSTM}}
        & \textbf{2} & 0.433 & 0.501 & 0.137 & 0.291 & 0.583 & 0.121 & 0.423 & 0.311 & 0.623 \\
        & \textbf{3} & 0.396 & 0.532 & 0.230 & 0.291 & 0.488 & 0.119 & 0.460 & 0.274 & 0.702 \\
        & \textbf{4} & 0.412 & 0.546 & 0.236 & 0.311 & 0.477 & 0.122 & 0.474 & 0.283 & 0.702 \\
        & \textbf{5} & 0.415 & 0.554 & 0.234 & 0.338 & 0.472 & 0.120 & 0.466 & 0.283 & 0.686 \\
    \midrule
    \multirow[c]{4}{*}{\textsc{Transformer}}
        & \textbf{2} & 0.740 & 0.679 & 0.455 & 0.290 & 0.658 & 0.844 & 0.551 & 0.622 & 0.709 \\
        & \textbf{3} & 0.743 & 0.728 & 0.569 & 0.374 & 0.735 & 0.859 & 0.693 & 0.832 & 0.820 \\
        & \textbf{4} & 0.737 & 0.674 & 0.549 & 0.389 & 0.727 & 0.844 & 0.668 & 0.848 & 0.799 \\
        & \textbf{5} & 0.717 & 0.614 & 0.516 & 0.333 & 0.705 & 0.830 & 0.625 & 0.833 & 0.770 \\
    \bottomrule
    \end{tabular}

    \caption{Pearson correlation coefficients between \mlocal entropy and KL divergence for different architectures, number of states $\nstates$, and alphabet sizes $\nsymbols$.}
    \label{tab:mlocal_entropy_pfsa_correlation}
\end{table*}

\section{Additional Experiments with BabyLM Corpus}\label{sec:appendix-babylm}
We also conducted the same set of experiments using the BabyLM corpus~\cite{choshen2024callpapers2ndbabylm}.
\Cref{tab:local_entropy_babylm} and \Cref{fig:lm_performance_babylm} show our experimental results.
They show the same trends as in our main experiment (\Cref{sec:exp_natural}), but with slightly different tendencies for the \textsc{Reverse} language.

\begin{table*}[h]
    \centering
    \begin{tabular}{ccccc}
    \toprule
     & 2-local entropy & 3-local entropy & 4-local entropy & 5-local entropy \\
    \midrule
    \textsc{Base}                & 5.78      & 3.72      & 2.69      & 2.32      \\
    \midrule
    \textsc{Reverse}             & 6.50      & 3.87      & 2.78      & 2.43      \\
    \textsc{EvenOddShuffle}      & 6.82      & 4.47      & 3.36      & 3.14      \\
    \textsc{OddEvenShuffle}      & 6.94      & 4.45      & 3.39      & 3.14      \\
    \textsc{LocalShuffle (k=3)} & 6.99 ± 0.16 & 4.27 ± 0.08 & 3.21 ± 0.07 & 2.97 ± 0.08 \\
    \textsc{LocalShuffle (k=4)} & 7.09 ± 0.15 & 4.35 ± 0.06 & 3.25 ± 0.03 & 3.06 ± 0.04 \\
    \textsc{LocalShuffle (k=5)} & 7.15 ± 0.13 & 4.42 ± 0.06 & 3.29 ± 0.04 & 3.08 ± 0.03 \\
    \textsc{LocalShuffle (k=6)} & 7.25 ± 0.11 & 4.47 ± 0.07 & 3.35 ± 0.06 & 3.14 ± 0.05 \\
    \textsc{LocalShuffle (k=7)} & 7.28 ± 0.12 & 4.50 ± 0.08 & 3.39 ± 0.07 & 3.19 ± 0.08 \\
    \textsc{DeterministicShuffle} & 7.41      & 4.69      & 3.59      & 3.40      \\
    \bottomrule
    \end{tabular}
    \caption{M-local entropy values for \textsc{Base} (original) corpus and different transformed corpora. ``Local shuffle'' refers to the \textsc{K-localDeterministicShuffle}. Values are shown as mean ± standard deviation (averaged over different random seeds).}
    \label{tab:local_entropy_babylm}
\end{table*}

\begin{figure*}
    \centering
    \includegraphics[scale=0.23]{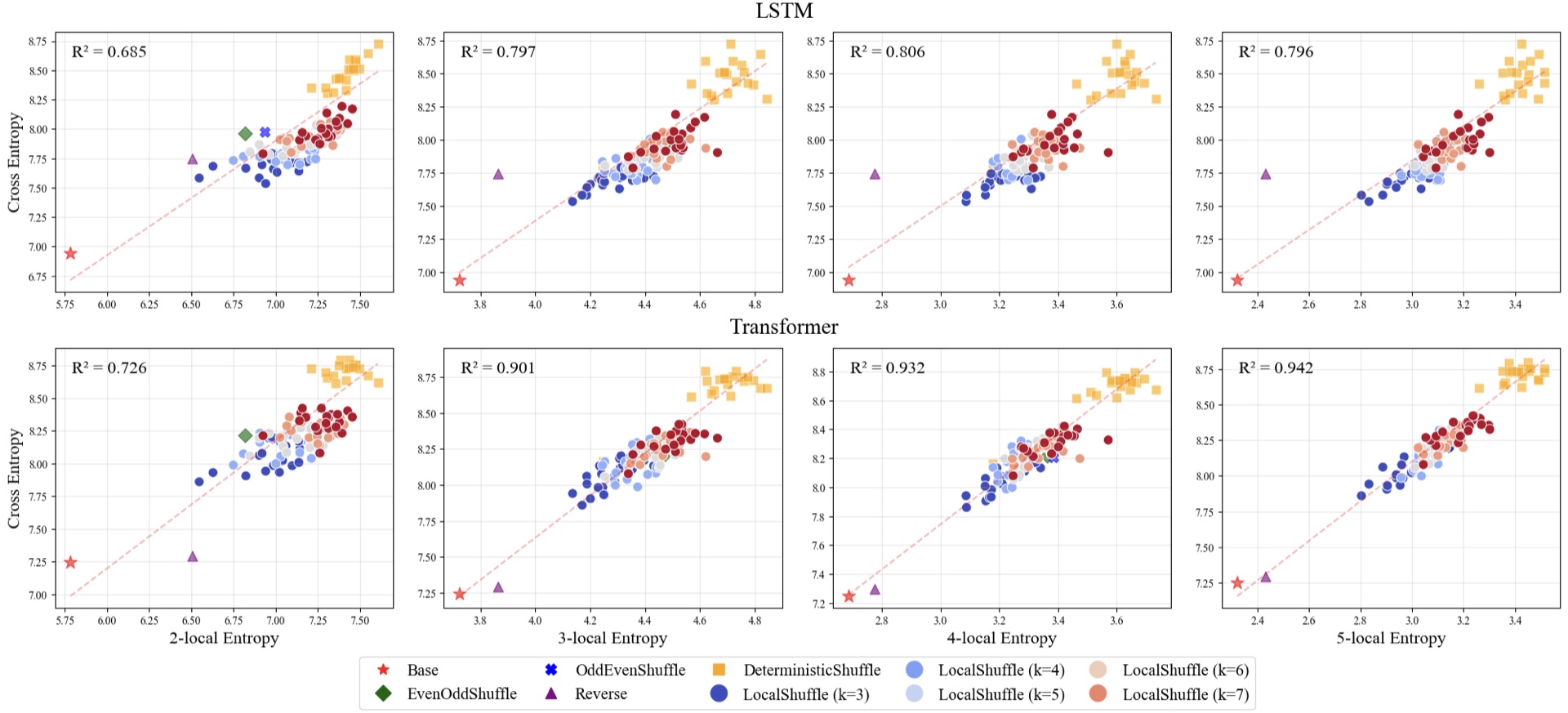}
    \caption{Scatter plots of next-symbol cross-entropy (y-axis) versus \mlocal entropy (x-axis) for $m \in \{2,3,4,5\}$, for both LSTM (top row) and causally-masked Transformer encoder (Transformer; bottom row) models.
    Each marker type/color corresponds to a different perturbation (e.g., \texttt{Reverse}, \texttt{DeterministicShuffle}, \textsc{K-localDeterministicShuffle} with various window sizes, etc.).
    The red star indicates the unperturbed \textsc{Base} condition (original corpus).
    The dashed line in each panel is a linear fit, with $R^2$ indicating the coefficient of determination.}
    \label{fig:lm_performance_babylm}
    \vspace{-15pt}
\end{figure*}

\section{Computational Resources}
Across all experiments, we used a total of approximately 717.5 GPU hours.
Training was conducted on NVIDIA GeForce RTX 4090 24GB and NVIDIA Quadro RTX 6000 24GB GPUs.

\section{License of the Data}
The BLLIP corpus~\citep{charniakeugene2000BLLIP} is used under the terms of the BLLIP 1987-89 WSJ Corpus Release 1 License Agreement.

\end{document}